\documentclass{article}       
\usepackage{graphicx}
\usepackage{tabularx,colortbl}
\usepackage[table]{xcolor}
\usepackage{amssymb}
\usepackage{subfigure}
 \usepackage{caption}
 \usepackage{scalerel}

\def\startj{\color{black}}
\def\endj{\color{black}}
\usepackage{amsfonts}
\usepackage{amssymb}
\usepackage{amsmath}
\usepackage{color}

\newcommand{\Omit}[1]{}

\newcommand{\V}{V}

\newcommand{\Jv}{J}

\newcommand{\Pv}{P(V)}

\newcommand{\GM}{M}

\newcommand{\RMWAO}{{{\textsc{mwa}}}}
\newcommand{\RYO}{{{\textsc{y}}}}

\newcommand{\RMSAO}{{{\textsc{msa}}}}

\newcommand{\RMCSAO}{{{\textsc{mcsa}}}}
\newcommand{\RRAO}{{{\textsc{ra}}}}

\newcommand\marijan[1]{{\color{black} #1}}
\newcommand\marija[1]{{\color{black} #1}}
\newcommand\ms[1]{{\color{black} #1}}
\newcommand\jl[1]{{\color{black} #1}}
\newcommand\gabri[1]{{\color{black} #1}}
\newcommand\srdjan[1]{{\color{black} #1}}
\newcommand\lella[1]{{\color{black} #1}}
\newcommand\leon[1]{{\color{black} #1}}

\definecolor{light-gray}{gray}{0.95}

\newcommand{\Winners}{Win}

\def\sj{\color{black}}
\def\ej{\color{black}}

\usepackage{amsfonts}
\usepackage{amssymb}
\usepackage{amsmath}
\usepackage{color}

\newcommand{\La }{\mbox{$\mathcal{L}$} }       

\newcommand{\A}{\mathcal{A}}

\newcommand{\PA}{[\mathcal{A}]}

\newcommand{\inc}{\mathtt{~inc~}}

\newcommand{\ai}{\varphi}

\newcommand{\Js}{J}

\newcommand{\Ct}{\gamma}

\newcommand{\Dma}{\mathcal{J}_{\A}}

\newcommand{\Dmc}{\mathcal{J}}

\newcommand{\Pf}{P}

\newcommand{\Qf}{Q}

\newcommand{\Alt}{C}

\newcommand{\F}{F}
\newcommand{\R}{F}

 \newcommand{\ie}{\mbox{\textit{i.e.}, }}
\newcommand{\eg}{\mbox{\textit{e.g.}, }}

\newcommand \argmax[1] {\underset{{#1}}{\mbox{argmax}}}
\newcommand \argmin[1] {\underset{{#1}}{\mbox{argmin}}}
\newcommand \mmin[1] {\underset{{#1}}{\mbox{min}}}
\newcommand \mmax[1] {\underset{{#1}}{\mbox{max}}}
\newcommand \ssum[1] {\underset{{#1}}{\mbox{$\sum$}}}

\newcommand{\ssu}{\succsim_{\Pf}}

 \newcommand{\npf}[2]{N({#1},{#2})}

\newcommand{\RMWA}{{{\textsc{med}}}}
\newcommand{\RY}{{{\textsc{y}}}}
\newcommand{\RMSA}{{{\textsc{mc}}}}
\newcommand{\RMNAC}{{{\textsc{mpc}}}}
\newcommand{\RMCSA}{{{\textsc{mcc}}}}
\newcommand{\RRA}{{{\textsc{ra}}}}
\newcommand{\RDG}{{\F^{d_G,\Sigma}}}
\newcommand{\RMAX}{{\F^{d_H,\textsc{max}}}}
\newcommand{\RSUM}{{\F^{d_H,\Sigma}}}
\newcommand{\RS}{\F_{s}}
\newcommand{\REVS}{\F_{\mbox{rev}}}

\newcommand{\RLEXIMAX}{{\textsc{leximax}}}

\newcommand{\MC}[1]{{max(#1, \subseteq)}}

\newcommand{\MCC}[1]{{max(#1, |.|)}}

\definecolor{light-gray}{gray}{0.95}
\definecolor{LightCyan}{rgb}{0.88,1,1}

\newtheorem{definition}{Definition}
\newtheorem{example}{Example}

\newtheorem{proposition}{Proposition}
\newtheorem{corollary}{Corollary}
\newtheorem{lemma}{Lemma}
\newenvironment{proof}{\noindent\emph{Proof.}}{\hfill $\Box$\break\par}

\begin{document}

\title{A \leon{partial} taxonomy of judgment aggregation rules and their properties \thanks{
}
}

\author{J\'er\^ome Lang \and Gabriella Pigozzi \and Marija Slavkovik \and Leendert van der Torre \and Srdjan Vesic}

%

\maketitle

\begin{abstract}
The literature on judgment aggregation is   moving from studying impossibility results regarding aggregation rules towards studying specific judgment aggregation rules. Here we give a structured list of most rules that have been proposed and studied recently in the literature,  together with various properties of such rules. We first focus on  the  majority-preservation  property, which generalizes Condorcet-consistency, and identify which of  the rules satisfy it. We study the inclusion relationships that hold between  the  rules. Finally, we consider two forms of unanimity, monotonicity, homogeneity, and reinforcement, and   we identify which of  the rules satisfy these properties.
\end{abstract}


\section{Introduction}\label{sec:introduction}
   
Judgment aggregation studies the problems related to aggregating  a finite set of \gabri{yes-no} individual judgments, cast on a 
collection of logically interrelated issues. \gabri{Such a finite set of issues forms the \emph{agenda}}.
It can be seen as a generalisation of preference aggregation \cite{Dietrich07}. 

Until a few years ago, the judgment aggregation literature had focused considerably more on studying  impossibility \lella{theorems} 
than on  developing and \gabri{investigating} specific  aggregation rules.
\leon{This} field development approach  departs from the, admittedly much older, field of voting theory. Nevertheless, several recent \gabri{and independent} \leon{papers} have started to explore the zoo of  concrete judgment aggregation rules, beyond the well known premise-based and conclusion-based rules \ms{\cite{premise10,ADT09}}. While the premise- and conclusion-based  rules  can only be applied if there exists  a prior labelling of the agenda issues as premises and conclusions, the following  rules are defined for \leon{every} agenda: quota-based rules \cite{DietrichList07}, distance-based rules \cite{Pigozzi2006,MillerOsherson08,EndrissGP12,DuddyP:2012}, generalizations of Condorcet-consistent voting rules \cite{PuppeNehring2011,NehringPivato2011,TARK11}, \leon{and} rules  based on the maximisation of some scoring function \cite{TARK11,Dietrich:2013,Zwicker11}. 
Some of these rules 
obviously generalize well-known voting rules.
\jl{However,  a `compendium' of existing judgment aggregation rules really does not exist at the moment, despite the several overview papers, chapters and even books that have been published in recent years \cite{ListPuppe2009,GrossiPigozzi14,EndrissHBCOMSOC2016,BaumeisterErdelyiRothe16}.}

Our aim is threefold. 
First, as there is so far no compendium of judgment aggregation rules, we give one: we list most of the rules that have been proposed recently, in a structured way.
This part of the paper does not 
give 
novel results, but serves as a partial survey.
 Second, we compare in a systematic way these rules in terms of inclusion relationships. Third, we consider a few key properties that generalize properties of voting rules (majority-preservation, unanimity, monotonicity, homogeneity and reinforcement) and identify those of the considered rules that satisfy them.

We follow earlier work in judgment aggregation \cite{ListPettit04} in using a constraint-based version of judgment aggregation to represent properties like transitivity of preferences.  As it is common in voting theory, we consider irresolute rules (also called `correspondences') rather than functions, that is, a rule outputs a non-empty set of collective judgments. 

The outline of the paper is as follows. The general definitions are given in Section \ref{sec:definitions}. 
In Section \ref{sec:rules} 
we review the rules we study in the paper.
Majority preservation is a key property of rules, \gabri{as it generalizes} Condorcet-consistency. 
We focus on \ms{majority-preservation} in Section \ref{sec:mp} and show which of \gabri{the rules defined in Section \ref{sec:rules}} satisfy it.
In Section \ref{sec:relation} 
we address inclusion and non-inclusion relationships between our rules.
In Section \ref{prop} 
we study the rules from the point of view of 
unanimity, monotonicity, reinforcement and homogeneity. 
We summarize our contributions in Section \ref{sec:conclu}. 

\section{Preliminaries}\label{sec:definitions}
 

  
  
Let $\La$ be a \leon{standard propositional language consisting of} well-formed propositional logical formulas, including $\top$ (tautology) and $\bot$ (contradiction), \leon{together with a standard notion of logical consistency}.  We denote atomic propositions by $p,q$ etc. and formulas from $\La$ by $\varphi, \alpha, \beta$ etc. 


An {\em agenda} $\A$ is a finite set of propositions of the form 
$\{\ai_1,\neg \ai_1, \ldots, \ai_m,$ $ \neg \ai_m\}$, where for all $i$, $\ai_i \in\La$ and $\ai_i$ is neither a tautology nor a contradiction, 
\sj and is a non-negated formula, (\ie \ms{it} is not \ms{of} the form $\neg \alpha$).\ej
We refer to a pair $(\ai, \neg \ai)$ as an {\em issue}. The {\em pre-agenda} $\PA$ associated with $\A$ is  $\PA = \{ \ai_1, \ldots, \ai_m\}$. 
We slightly abuse notation and write $\ai_i$ instead of $\neg \neg \ai_i$ for $\varphi_i \in \PA$.


 
An  agenda is {\em endowed with a notion of consistency} which preserves logical consistency. Formally, $\A$ comes with a set of {\em ($\A$-)consistent} judgment sets; \leon{a} 
\gabri{($\A$-)consistent} judgment set is logically consistent, but the converse does not necessarily hold. Without loss of generality, the agenda's consistency notion is defined as {\em logical consistency given some fixed formula}: a set of formulas $S$ is consistent if $S \cup \{ \Ct \}$ is logically consistent, where $\Ct$ is some exogenously fixed non-contradictory formula, which we call the {\em integrity constraint}.  
\jl{This is also the approach followed in \cite{GrandiE13} (albeit in the slightly different framework of binary aggregation, where agenda issues are atomic propositions) and in \cite{EndrissGHL16}. A similar use of constraints is also done in belief merging theory \cite{KPP02,EveraereKM15}.}
When $\Ct$ is not specified, by default it is equal to $\top$, in which case the notion of consistency associated with the agenda \leon{coincides with} 
standard \leon{logical} consistency.

 A {\em  judgment} on $\ai \in \PA$ is \leon{either} 
 $\ai$ or $\neg \ai$.  A {\em judgment set} $\Js$ \leon{for $\A$} is a subset of $\A$.
$\Js$ is {\em complete} if and only if  for each $\ai\in\PA$, either $\ai\in\Js$ or $\neg\ai\in\Js$.
A judgment set for $\A$ is {\em rational} if it is complete and consistent. Let $\Dmc_{\A}$ be the set of all rational judgment sets for $\A$.

For every consistent $S \subseteq \A$, the set of rational extensions of $S$, \ie $\{ \Js \mid \Js \in \Dmc_{\A} \mbox{ and } S \subseteq \Js \}$, is denoted as $ext(S)$.

A {\em $\Dma$-profile}, or simply a {\em profile}, is a \leon{finite sequence} 
of rational individual judgment sets, \ie $\Pf =\langle\Js_1, \ldots, \Js_n \rangle$ for some $n$, where $\Js_i$ is the judgment set of voter $i$.
We slightly abuse notation and write $\Js \in \Pf$ when
$\Js = \Js_i$ for some $i$, and we write $|\Pf|$ to denote the number of judgment sets in $\Pf$. We sometimes denote $\Pf$  as $(\Js_i, \Js_{-i})$, where $\Js_{-i} = \langle \Js_j, 1 \leq j \leq n, j \neq i \rangle$.  We write $\Qf \subseteq \Pf$ (read "$\Qf$ is a sub-profile of $\Pf$) if $\Qf = \langle \Js_j | i \in I \rangle$ for some $I \subseteq \{1, \ldots, n\}$.

Given two rational judgment sets $\Js$ and $\Js'$ we define the {\em Hamming distance} $d_H$: $d_H(\Js, \Js')$ \gabri{as} the number of issues on which $\Js$ and $\Js'$ disagree.
We also define the Hamming distance between two profiles $\Pf = \ms{\langle}\Js_1, \ldots, \Js_n\ms{\rangle}$ and $\Pf' = \ms{\langle}\Js_1', \ldots, \Js_n'\ms{\rangle}$ as 
$D_H(\Pf,\Pf') = \sum_{i=1}^n d_H(\Js_i,\Js'_i)$, and between a judgment set and a profile as
$d_H(\Js,\Pf) = \sum_{i=1}^n d_H(\Js,\Js_i)$.

We define $N(\Pf,\ai)$ as the number of all  voters  in $\Pf$  whose judgment set contains $\varphi$, \ie $\npf{\Pf}{\ai}  =  | \{ i \mid \varphi \in \Js_i, \Js_i \in \Pf\}|$. 

 \begin{example}\label{ex:agex} Consider the pre-agenda $\PA = \{p\wedge r, q,p\wedge q\}$. 
The corresponding agenda is $\A  = \{p\wedge r, \neg (p\wedge r),   q, \neg q, p\wedge q, \neg (p\wedge q)\}$, equipped with the consistency notion \gabri{corresponding to} 
 $\Ct = (q \rightarrow r)$. 
The set  
of rational judgment sets is 
\[\Dma =  
\begin{Bmatrix}
\{\neg (p\wedge r), \neg q, \neg (p\wedge q)\}, & \{\neg (p\wedge r),   q, \neg (p\wedge q)\}, \\
 \{ (p\wedge r), \neg q, \neg (p\wedge q)\}, & \{ p\wedge r,   q,  p\wedge q \}\\
\end{Bmatrix}
\] 
Consider the profile $\Pf = \langle \Js_1, \Js_2, \Js_3, \Js_4\rangle$ with $\Js_1 = \Js_2  = \{\neg (p\wedge r),  q, \neg (p\wedge q)\}$, $\Js_3 = \{p\wedge r, \neg q, \neg (p\wedge q)\}$ and $\Js_4 = \{p\wedge r,  q,  p\wedge q \}$. For instance, $\npf{\Pf}{q}= 3$.

%
%
\end{example}

Most often we will write profiles in a table, as in Table \ref{tab:maj}, 
with the pre-agenda elements given in the topmost row and the \gabri{voters'}  judgment sets in the leftmost column. If a judgment set contains $\varphi \in \PA$ (respectively $\neg \varphi$), then we mark this with a ``+" (respectively,  ``-") in the \gabri{corresponding column}.
The constraint, if explicitly defined, will be  given in the table caption.

The {\em majoritarian judgment set} associated with profile $\Pf=\langle\Js_1, \ldots, \Js_n \rangle $ contains all the elements of the agenda that are supported by a strict majority of judgment sets in $\Pf$, 
\ie $m(\Pf)=\left\{\ai\in\A\mid \npf{\Pf}{ \ai} >\frac{n}{2}\right\}.$
A profile $\Pf$ is {\em  majority-consistent}
when $m(\Pf)$ is a consistent subset of $\A$.

An (irresolute) {\em judgment aggregation rule} $\F$ maps every profile $\Pf$, defined on \leon{every} 
agenda $\A$, to a nonempty set of rational judgment sets in $\Dmc_{\A}$.
When for all 
profiles $\Pf$,  $\F(\Pf)$ is a singleton, then  $\F$ is said to be {\em resolute}. Like in voting theory, resolute rules can be defined from irresolute ones by \gabri{coupling} them with a tie-breaking mechanism.

The {\em preference pre-agendas} associated with a set $\Alt=\{x_1,\ldots, x_q\}$ of alternatives are defined by the set of \leon{atomic}
propositions $\{ x_iPx_j\mid1\leq i < j\leq q\}$
(when $j>i$, $x_j P x_i$ is not a\leon{n atomic} proposition, but we  write $x_j P x_i$ as a shorthand for $\neg (x_iPx_j)$)
and one of these two consistency notions:
{\em transitivity}, defined as consistency with $$Tr = \bigwedge_{\underset{i\neq j\neq k}{i,j,k \in \{1,\ldots,q\}} } \big(x_iPx_j \wedge x_jPx_k \rightarrow x_iPx_k\big)$$
or {\em existence of a nondominated alternative},
defined as consistency with $$W = \bigvee _{i \in \{1, \ldots, q\}}\bigwedge_{j\neq i} \ x_iP x_j$$
Finally, $\A_{\Alt}^{Tr}$ (respectively, $\A_{\Alt}^W$) is defined by its pre-agenda $[\A_{\Alt}^{Tr}] = \{ x_iPx_j\mid1\leq i < j\leq q\}$ and the consistency notion corresponding to transitivity (respectively, with the existence of a nondominated alternative).

A {\em preference profile} over $\Alt$ is a \leon{finite sequence} 
of linear orders over $\Alt$, which we denote by $\V = \langle \succ_1, \ldots, \succ_n\rangle$. The majority graph 
$M(\V)$ associated with $\V$ is the directed graph whose vertices are elements of $\Alt$ and containing edge $(x,y)$ if and only if a majority of voters in $\V$ prefer $x$ to $y$; we denote by $n_{\V}(x,y)$ the number of votes in $\V$ that prefer $x$ to $y$. A {\em social preference function} maps \leon{every} 
preference profile to a nonempty set of linear orders over $\Alt$. A {\em social choice function} (or {\em voting rule}) maps \leon{every} 
preference profile to a nonempty subset of $\Alt$. With \leon{every} 
judgment aggregation rule we can associate {\em two} social preference functions, whether we impose the transitivity constraint or the nondominated alternative constraint. From these two social preference functions we can derive two social choice functions by ``collecting'' the nondominated elements in each of the output preference relations. Sometimes, especially when $n$ is odd, these social preference functions or the social choice functions coincide with well-known \gabri{voting} rules (we \ms{show}  several examples in Section \ref{sec:rules}).

\section{Judgment aggregation rules}\label{sec:rules}

We now define five (overlapping) families of judgment aggregation rules.
We use the following running example throughout the paper to illustrate the rules.
\begin{example}\label{ex:maj}
Let $ \PA = \{p\wedge r,p \wedge s,q,p\wedge q,t\}$, and $\Pf$ the 17-voter profile   of Table \ref{tab:maj}.
Consistency in $\A$ is logical consistency.
As $m(\Pf) =\{ p\wedge r,p \wedge s, q, \neg(p\wedge q), t \}$ is an inconsistent judgment set,  $\Pf$ is not majority-consistent. 
 
\begin{table}[h!]
\begin{center}
\begin{tabular}{r|ccccc}
  Voters       &\{ $\;\;\; p\wedge r $,& $\;\;\;  p \wedge s$, &  $\;\;\; q$, &   $\;\;p\wedge q$, &   $\;\;\;t$ \} \\\hline
  $\Js_1 -\Js_6$          &  $\;\;\;\;$+           &  $\;\;\;$+             &  $\;\;\;$+   &  $\;$ +           &  +    \\
  $\Js_7 -\Js_{10}$        & $\;\;\;\;$+           & $\;\;\;$+             & $\;\;\;$ -   &  $\;$ -           &  +    \\
 $\Js_{11} -\Js_{17}$        & $\;\;\;\;$-           & $\;\;\;$-             &   $\;\;\;$+   &  $\;$ -           &  -    \\\hline
m(\Pf) & $\;\;\;$ + & $\;\;\;\;$+ & $\;\;\;\;$+ & $\;\;\;$- &  +
\end{tabular}
\caption{Running example profile.}
 \label{tab:maj}
\end{center}
\end{table}
\end{example}

\subsection{Rules based on the majoritarian judgment set}\label{family-majoritarian}

A judgment aggregation rule $\F$ is {\em based on the majoritarian judgment set} when for \leon{every} 
two $\Dmc_{\A}$-profiles   $\Pf$ and $\Pf'$ such that $m(\Pf) = m(\Pf')$, we have $\F(\Pf) = \F(\Pf')$. 
\leon{These rules} 
can be viewed as the judgment aggregation counterpart\leon{s} of voting rules based on the pairwise majority graph, also known as {\em C1 rules} in Fishburn's classification (see, \eg \cite{BramsFishburn04}).


Given a set of formulas $S\subseteq \A$,
$S' \subseteq S$ is a {\em maximal consistent subset of $S$} if and only if $S'$ is consistent and there exists no other consistent set $S''$ such that $S' \subset S''  \subseteq S$; and $S' \subseteq S$ is a {\em maxcard} (for ``maximal cardinality")
{\em consistent  subset of $S$} if and only if $S'$ is consistent and there exists no other consistent set $S'' \subseteq S$ such that $|S''| > |S'|$.
The set of maximal  (respectively, maxcard) consistent subsets of $S$ is denoted \endj by $\MC{S}$ (respectively, $\MCC{S}$).
 
 \begin{definition}[Maximal Condorcet and maxcard Condorcet rules]
 \label{def:msa}   
 For \leon{every} 
    $\Dma$-profile $\Pf$,  the \lella{maximal} Condorcet  rule ($\RMSA$) and the maxcard Condorcet rule ($\RMCSA$)  are defined as follows:
 \begin{eqnarray}
 \RMSA(\Pf) = \{ext(S) \mid S\in\MC{m(\Pf)}\},  \\ 
 \RMCSA(\Pf) = \{ext(S) \mid S\in\MCC{m(\Pf)}\}. 
\end{eqnarray}
\end{definition}

\noindent Equivalently, $\RMCSA(\Pf) = \argmax{J \in \Dma} |J \cap m(\Pf)|$.
Clearly, $\RMCSA(\Pf) \subseteq \RMSA(\Pf)$. 

\begin{example}\label{ex:msa}
For the profile $\Pf$ of Example \ref{ex:maj}, 
the maximal  consistent subsets of $m(\Pf)$ are $\{p \wedge r, p \wedge s, q, t\}$, $\{p \wedge r, p \wedge s, \neg(p \wedge q), t\}$ and $\{q, \neg(p \wedge q), t\}$; therefore  
\vspace{-1mm}
\[\RMSA(\Pf) =
\left\{
\begin{array}{ccccc}
 \{p\wedge r,&p\wedge s,&q,  &                 p\wedge q,                 &t\},       \\
\{p\wedge r, &p\wedge s, & \neg q,   & \neg (p\wedge q),&t\},   \\
\{     \neg (p \wedge r), &  \neg (p \wedge s),  & q, &\neg(p\wedge q),&t\} 
\end{array}
\right\}
\]
\vspace{-1mm}
and 
\vspace{-1mm}
\[\RMCSA(\Pf) =
\left\{
\begin{array}{ccccc}
 \{p\wedge r, &p\wedge s, &q,&   p \wedge q,          & t\},       \\
\{p\wedge r,&p\wedge s, &  \neg q,   & \neg(p\wedge q),&t\}
\end{array}
\right\}.
\]
\end{example}

The output of the rule $\RMSA$ is called {\em Condorcet admissible set}  by Nehring et al. \cite{PuppeNehring2011}.
The rule $\RMCSA$ is called {\em Slater} rule  \cite{PuppeNehring2011}, and  {\sc Endpoint}$_{d_H}$  \cite{MillerOsherson08}.

%


{\em At least when $n$ is odd}, \ms{it} is easy to identify the voting rules obtained from $\RMSA$ and $\RMCSA$. We give these 
results informally and without proof:\footnote{The proofs can be found in \cite{ADT2013}.}
\begin{itemize}
\item For $\RMCSA$, 
the transitivity constraint leads to the social preference function that maps a profile to the set of all its {\em Slater orders}, \ie  the set of all  linear orders $\succ$ over $\Alt$ maximising the number of $(x,y)$ such that $x \succ y$ if and only if $(x,y) \in M(\V)$, and that the corresponding voting rule (for $n$ odd) is the {\em Slater rule}, which maps a  profile $\V$ to the set of all alternatives that are dominating in some Slater order for $M(\V)$. If we choose the $W$ constraint, then the corresponding voting rule (for $n$ odd) is the {\em Copeland rule}, which maps a profile to the set of alternatives $x$ maximising  the number of outgoing edges from $x$ in $\GM(\V)$.
\item For $\RMSA$, the transitivity constraint (for $n$ odd) leads to the {\em top cycle} rule, which maps a preference profile to the (unique)
smallest subset $S$ of $C$ such that for every $x \in S$ and $y \in \Alt\setminus S$, we have $(x,y) \in M(\V)$.\footnote{This result has been independently proven (and stated in a  stronger way) in \cite{PuppeNehring2011}.}
Finally, the choice of the $W$ constraint (for $n$ odd) leads to the voting rule that maps a profile to its Condorcet winner if and only if the profile has a Condorcet winner, and to the set of all alternatives otherwise.
\end{itemize}

%
%

\subsection{Rules based on the weighted majoritarian set}\label{family-weighted-majoritarian}

The {\em weighted majoritarian set} associated with a profile $\Pf$ is the function $N(\Pf, .)$ \jl{which, we recall, maps each agenda issue to the number of judgment sets in $\Pf$ that contain it}. A judgment aggregation rule $\F$ is {\em based on the weighted majoritarian set} when for \leon{every} 
two $\Dmc_{\A}$-profiles   $\Pf$ and $\Pf'$,  if \leon{for every $\varphi \in \A$ we have} $N(\Pf, \varphi) = N(\Pf', \varphi)$,  then $\F(\Pf) = \F(\Pf')$.
\leon{These rules} 
can be viewed as the judgment aggregation counterpart\leon{s} of voting rules that are based on the weighted pairwise majority graph, also known as {\em C2 rules} in Fishburn's classification \cite{BramsFishburn04}. 
Since $m(\Pf)$ can be recovered from $N(\Pf,.)$,  
\leon{every} 
rule based on the majoritarian judgment set is also based on the weighted majoritarian set.


\begin{definition}[Median rule]\label{def:mwa}  For \leon{every} 
$\Dma$-profile $\Pf$,
the median rule ($\RMWA$) is defined as follows: 
\begin{equation}\label{eq:mwa}
\RMWA(\Pf) = \argmax{\Js \in \Dma}  \sum_{\ai \in \Js} \npf{\Pf}{\ai}.  \end{equation}
\end{definition} 

This rule appears in  many places under different names: {\sc Prototype} \cite{MillerOsherson08}, {\em median rule} \cite{PuppeNehring2011}, {\em maximum weighted agenda rule} \cite{TARK11}, {\em simple scoring rule} \cite{Dietrich:2013} and {\em distance-based procedure} \cite{EndrissGP12}. Variants of this rule have been defined by Konieczny and Pino-P\' erez \cite{KPP02} and Pigozzi \cite{Pigozzi2006}. For  completeness we give here \leon{the} 
equivalent distance-based formulation of $\RMWA$, although we consider more generally the family of distance-based rules in Section~\ref{sec:dist}. For \leon{every} 
$\Dma$-profile $\Pf$, 
the distance-based rule $\RSUM$ is defined as follows:   

\begin{equation}\label{eq:rsum}  \RSUM(\Pf) = \argmin{\Js \in \Dma} \sum_{\Js_i \in \Pf} d_H(\Js_i, \Js).\end{equation}

It is not difficult to establish that $\RSUM$ coincides with $\RMWA$ (see \cite{TARK11}, and Proposition 1 in \cite{Dietrich:2013}). 
The social preference function
obtained from $\RMWA$ and the choice of the transitivity constraint is the {\em Kemeny social preference function}, 
and the corresponding voting rule is the {\em Kemeny rule}.
 
\begin{example}\label{ex:mwa}
Consider the agenda and profile of Example \ref{ex:maj}. We obtain:\smallskip

\hspace{-7mm} \begin{tabular}{llll}
$N(\Pf,p\wedge r)\negthinspace=\negthinspace10$ & $N(\Pf, \neg (p\wedge r))\negthinspace=\negthinspace7$ & $N(\Pf,p\wedge s)\negthinspace=\negthinspace10$ &$N(\Pf, \neg (p\wedge s))\negthinspace=\negthinspace7$\\
$N(\Pf,q) =13$ &  $N(\Pf, \neg q) =4$ & $N(\Pf,p\wedge q)\negthinspace=\negthinspace6$ & $N(\Pf, \neg (p\wedge q))\negthinspace=\negthinspace11$ \\
$N(\Pf,t) =10$ & $N(\Pf, \neg t) =7$ \\
\end{tabular}\smallskip

\noindent As $\ssum{\ai \in \Js} \npf{\Pf}{\ai}$ reaches its maximum value (49) for  $\{p\wedge r,p\wedge s,q,p\wedge q,t\}$, we have
$\RMWA(\Pf) =\{\{p\wedge r,p\wedge s,q,p\wedge q,t\}\}$.
\end{example}


The following rule 
generalizes the {\em ranked pairs} voting rule \cite{Tideman87}. 
It proceeds  by considering the elements $\varphi$ of the agenda in non-increasing order of $N(\Pf,\varphi)$  
and fixing each agenda issue value to the majoritarian value if it does not lead to an inconsistency. 
 


\begin{definition}[Ranked agenda rule]~\label{def:ra}
%
%
%
Let $\A = \{\psi_1, \ldots, \psi_{2m}\}$. For \leon{every} 
$\Dma$-profile $\Pf$, $\RRA$ consists of those judgment sets
$\Js \in \Dma $ for which there exists a permutation $(\ai_1, \ai_2,  \ldots, \varphi_{2m})$ of the propositions in
$\A$ such that $\npf{\Pf}{\ai_1} \geq\npf{\Pf}{\ai_2} \geq \cdots\geq \npf{\Pf}{\ai_{2m}} $ and $\Js$ is obtained by the
following  algorithmic procedure:
\begin{center}
\begin{tabular}{l}
$S := \emptyset$\\
$\mbox{\bf for } k = 1, \ldots,2m\mbox{ do}$\\
~~~~$\mbox{\bf if } S \cup \{ \ai_k\} \mbox{ is consistent}$ $ \mbox{\bf then } S \leftarrow S \cup \{ \ai_k\}$ \\
~~~~$\mbox{\bf end if }$\\
{\bf end for}\\
$\Js := S$
\end{tabular}
\end{center}
\end{definition}

In plain words, $\RRA$ assigns iteratively a truth value to each proposition of the agenda, whenever it does not produce an inconsistency with propositions already assigned, following an order compatible with $\npf{\Pf}{.}$. An equivalent non-procedural definition is the following: for \leon{every} 
profile $P$, define $>_P^{RA}$ by: $J >_{\Pf}^{RA} J'$ \startj if there is an $\alpha \in {\mathbb N}$ such that 
\begin{enumerate}
\item for all $\psi \in \A$, $N(\Pf, \psi) > \alpha$ implies [$\psi \in J$ if and only if $\psi \in J'$], \leon{and}
\item $J \cap \{\varphi \ | \ N(P,\varphi) = \alpha \} \supset J' \cap \{\varphi \ | \ N(P,\varphi) = \alpha \}$.
\end{enumerate}
\endj
Then  $\RRA(\Pf) = \{J \in \Dma \ | \ J \mbox{ undominated in } >_{\Pf}^{RA} \}$.\footnote{\startj The proof---almost straightforward---can be found in \cite{Lang15}\endj.}


\begin{example}\label{ex:ra} 
Consider the  profile of Example \ref{ex:maj}.   The highest value of $\npf{\Pf}{\ai}$ is reached for $q$, therefore $q$ is fixed first. Then comes $\neg(p\wedge~q)$, which is fixed as well. Then come $p \wedge r$ and $p \wedge s$, tied. We skip both
because they would produce inconsistencies; then $t$ is fixed, and finally, $\neg(p \wedge r)$ and $\neg(p \wedge s)$. Thus,
\[\RRA(\Pf) = \{\{ q, \neg (p\wedge q),  t, \neg (p\wedge r),\neg (p\wedge s)\}\}.\]
\end{example}


\marija{The {\sc leximax} rule \cite{NehringPivato2011,EKM13} is a refinement of $\RRA$. We repeat the definition of {\sc leximax} here using our terminology. 
\begin{definition}
Given an $n$-voter profile $\Pf$ \startj and a rational judgment set $\Js$, define $S_k(\Pf) = \{ \ai \in \A \mid N(\Pf, \ai) = k, \frac{n}{2} \leq k \leq n\}$ and $s_k(\Js,\Pf) = |S_k(\Pf) \cap \Js |$.  Given two rational judgment sets $\Js, \Js'$, let $\Js >^{leximax}_\Pf \Js'$ if and only if there is a $k \in \{\frac{n}{2}, \ldots, n\}$ such that
$s_k(\Js, \Pf) > s_k(\Js',\Pf)$ and 
for all $i > k$,  $s_i(\Js,\Pf) = s_i(\Js',\Pf)$.
$\RLEXIMAX(\Pf)$ is the set of all undominated rational judgment sets in with respect to $>_\Pf^{leximax}$.
\end{definition}

For the profile $\Pf$ of Example \ref{ex:maj}: $S_{13}(\Pf) = \{ q\}$,  \startj $S_{12}(\Pf) = \emptyset$, \endj $S_{11}(\Pf) = \{ \neg (p\wedge q)\}$, $S_{10}(\Pf) = \{ p\wedge r, p\wedge s, t\}$. If \startj $\Js = \{ \neg (p\wedge r), \neg (p\wedge s), q, \neg (p\wedge q), t\}$ and $\Js' = \{ p\wedge r, p\wedge s, q, p\wedge q, t\}$, we have $s_{13}(\Js,\Pf) = s_{13}(\Js',\Pf) = 1$, $s_{12}(\Js,\Pf) = s_{12}(\Js',\Pf) = 0$, $s_{11}(\Js,\Pf) = 0$ and $s_{11}(\Js',\Pf) = 1$,
therefore  $\Js >^{leximax}_\Pf \Js'$. \endj
(In fact, $\Js$ is the only $>_\Pf^{leximax}$-undominated rational judgment set.)
}

It is easy to see that the social preference function (respectively, voting rule) associated with $\RRA$ and the transitivity constraint is the {\em ranked pairs} social preference function (respectively, rule), which informally proceeds by iteratively fixing edges in the majority graph, whenever possible, considering all ordered pairs of alternatives $(x,y)$ in an order corresponding to non-increasing values of $n_{\V}(x,y)$, and outputs the rankings obtained this way (respectively, the dominating elements in these rankings). \leon{However}, 
the voting rule associated with $\RRA$ and the $W$ constraint  is the {\em maximin} rule, that maps a profile $\V$ to  the set of alternatives that maximise $\mmin{y\in\Alt\setminus \{x\}} n_{\V}(x,y)$. The voting rules associated with $\RLEXIMAX$ are refinements of {\em ranked pairs} and {\em maximin}.
%

\subsection{Rules based on elementary changes in profiles}\label{family-individual}

The next family of rules we consider contains rules that are based on minimal set of changes on a profile needed to render the profile majority-consistent. This family of judgment rules can be viewed as the judgment aggregation counterpart of voting rules that \gabri{are} rationalisable by some distance with respect to the Condorcet consensus class  \cite{ElkindFaliszewskiSlinko09}.


The first rule we consider is called the {\em Young} rule for judgment aggregation, by analogy with the Young voting rule, which outputs the candidate $c$ minimising the number of voters to remove from the profile so that $c$ becomes a weak Condorcet winner \cite{Young77}. The judgment aggregation generalization consists of removing a minimal number of voters so that the profile becomes majority-consistent, or equivalently, to look for majority-consistent subprofiles of maximum cardinality.

\begin{definition}[Young rule]~\label{def:y} 
For \leon{every} 
$\Dma$-profile $\Pf$, 
\begin{equation}
\RY(\Pf) =  \{ ext(m(\Qf) ) \mid \Qf \in \underset{m(\Qf) \mbox{ is \jl{$\A$-}consistent}} {\argmax{\Qf \subseteq \Pf,}}|\Qf|
\}.\end{equation}
\end{definition}


\begin{example}\label{ex:y} Once again we consider $\A$ and $\Pf$ from Example \ref{ex:maj}. After noticing that removing 
three judgment sets from $\{\Js_1,\ldots,\Js_6\}$ restores majority-consistency, and
removing less than three judgment sets does not, we obtain 

\[\RY(\Pf) = ext(\{ q, \neg (p \wedge q) \})  =\left\{
\begin{array}{rrrrr}
 \{\neg (p\wedge r), & \neg (p\wedge s), & q,  &   \neg (p\wedge q), & t\},       \\
 \{\neg (p\wedge r), & \neg (p\wedge s), & q,  &   \neg (p\wedge q), & \neg t\}
\end{array}
\right\}.
\]
%
\end{example}


The voting rule associated with $\RY$ and the $W$ constraint is the Young voting rule \cite{ADT2013}.
\medskip

The \leon{next} 
rule we define  looks for a minimal number of individual {\em judgment reversals} in the profile so that $\Pf$ becomes majority-consistent, where a  judgment reversal is a change of truth value of one agenda element in one individual judgment set. 
This rule has been proposed first in  Miller and Osherson \cite{MillerOsherson08}  under the name \texttt{\sc full$_d$}.
It bears a resemblance with the {\em Dodgson} voting rule, but does not exactly \gabri{correspond} to it when choosing \ms{ either the $Tr$ or } the $W$ constraint. 
\marijan{
\begin{definition}[Minimal profile change rule]
 For $\Pf \in \Dma^n$,
  the $\RMNAC$ rule is defined as: 

$$\RMNAC(\Pf) =  \{ ext(m(\Qf)) \mid  \Qf \in \underset{m(\Qf) \textrm{ is $\A$-consistent}}{\argmin{ Q \in \Dma^n}} D_H(\Pf, \Qf) \}\lella{.}$$ 
\end{definition}
 }

\begin{example} \label{ex:mnac}~ 
  Consider the agenda $\A$ and profile $\Pf$ of Example \ref{ex:maj}. Profile  $\Pf'$  given in Table~\ref{tab:mnac} is the closest majority-consistent  profile to $\Pf$ with $D_H(\Pf,\Pf') = 3$ (the \lella{grey} cell indicates the judgments reversed from $\Pf$).
  We obtain $\RMNAC(\Pf) = \{\{p\wedge r, p\wedge s,q,p\wedge q, t\}\}$.\medskip
 
 \begin{table}[h!]
\begin{center}
\begin{tabular}{r|ccccc}
Voters  & $\{p\wedge r$ & $p\wedge s$ & $q$ & $p\wedge q$ & $t\}$\\\hline
$\Js_1 -\Js_6$& +               &+               &+ & +              &+ \\
$\Js_7 -\Js_{10}$& +               &+               &-  & -              &+ \\
$\Js_{11}-\Js_{14}$ & -                &-               &+  & -              &- \\
$\Js_{\marija{15}} - \Js_{17}$& -                &-               &+  &\cellcolor{gray!15} +              &- \\
\hline
$m(\Pf')$   & +               &+              &+  &+             &+
\end{tabular}
 \caption{ The profile at minimal $D_H$ distance from the profile $\Pf$ in Table~\ref{tab:maj}.}
\label{tab:mnac}
\end{center}
  \end{table}  
\vspace{-10mm}
\end{example}

\subsection{Rules based on (pseudo-)distances}\label{sec:dist}


For a given constrained agenda, a pseudo-distance $d$ on $\Dma$ is a function that maps pairs of judgment sets to non-negative real numbers, and that satisfies, for all $\Js, \Js' \in \Dma$, $d(\Js,\Js') = d(\Js',\Js)$, 
and $d(\Js,\Js')=0$ if and only if $\Js=\Js'$.

Two pseudo-distances we will use are the Hamming distance $d_H$, defined in Section \ref{sec:definitions}, and the {\em geodesic distance}\footnote{Our name; no name was given of this distance in \cite{DuddyP:2012}.}  on $\Dma$, defined in \cite{DuddyP:2012} as follows.
Given three  \leon{distinct} 
rational judgment sets $\Js, \Js', \Js''$,   we say that $\Js$ {\em is between} $\Js'$ and $\Js''$ if $\Js' \cap \Js'' \subset \Js$.  Let $G_{\A}$ be the graph whose 
set of vertices is the set of rational judgment sets $\Dma$ and that contains an edge between $\Js'$ and $\Js''$ if and only if there exists no \ms{ $\Js \in \Dma$,  $ \Js' \neq\Js \neq\Js''$,} between $\Js'$ and $\Js''$. Finally, $d_G(\Js', \Js'')$ is defined as the length of the shortest path between $\Js'$ and $\Js''$ in $G_{\A}$. 

\begin{definition} Let  $d$ be  a pseudo-distance on $\Dma$ and $\star$ a commutative, associative and non-decreasing function on $\mathbb{R}^+$.
The distance-based judgment aggregation rule $\F^{d,\star}$ associated with $d$ and $\star$ is defined as
\begin{equation}\F^{d,\star}(\Pf) = \argmin{\Js \in \Dma}\;  \startj \star (d(\Js_1, \Js), \ldots, d(\Js_n,\Js))\endj\end{equation}
\vspace{-4mm}
\end{definition}

In addition to $\RSUM$ we focus on two specific distance-based judgment aggregation rules: $\RDG$, defined in  \cite{DuddyP:2012}, and $\RMAX$, defined in \cite{KPP02,TARK11}. \startj From now on, we will use the word `distance' instead of `pseudo-distance' although our rules can be defined more generally for pseudo-distances.



\subsection{Scoring rules}\label{sec:scoring}


Dietrich \cite{Dietrich:2013} defines a general class of {\em scoring rules} for judgment aggregation. Given a function $s:  \Dma \times \A\rightarrow \startj \mathbb{R}^+$\endj, the rule $\RS$ is defined 
as \begin{equation}\RS(\Pf) = \argmax{\Js \in \Dma} \sum_{\varphi \in \Js}\sum_{\Js_i\in\Pf} s(\Js_i,\varphi).\end{equation} 
The $\RMWA$ rule (\ref{eq:mwa})
is a scoring rule (and also a distance-based rule). 
%
%
%

The reversal score function $\mbox{rev}$  \cite{Dietrich:2013} is defined as:

\begin{equation}\mbox{rev}(\Js,\varphi) = \mmin{\Js' \in \Dma, ~\varphi \notin \Js'} d_H(\Js,\Js'). \end{equation}

The main motivation for introducing this rule is that the associated voting rule (with the transitivity constraint) is the Borda rule.
Dietrich \cite{Dietrich:2013} defines four other scoring rules (entailment scoring, disjoint entailment scoring, minimal entailment scoring, and irreducible entailment scoring), \startj two of which \endj generalize the Borda rule as well.
As he 
focuses on reversal scoring, we do as well, and leave the other four for further study beyond this paper.

Duddy {\em et al.} \cite{DPZ15} introduce another interesting and intriguing scoring rule (defined only \startj when the agenda satisfies \endj a specific property); it generalizes not only the Borda rule, but also a well-behaved {\em mean} rule for finding collective dichotomies.
We leave it for further study as well.




\section{Majority-preservation}\label{sec:mp}

Intuitively, a judgment aggregation rule $\F$ is majority-preserving if and only if $\F$ returns only the  \leon{extensions of the majoritarian judgment set} whenever it is consistent.  
 \gabri{In case of ties, a majoritarian set can have more than one extension.} 
For example, when we have agenda $\A=\{p,\neg p,q,\neg q\}$ and  individual judgments $\Js_1=\{p,q\}$ and $\Js_2=\{p,\neg q\}$, then $m(\langle \Js_1, \Js_2 \rangle) = \{p\}$, which can be extended into two complete collective judgment sets, namely $\{p, \neg q\}$ and $\{p,q\}$.


\begin{definition}
A judgment aggregation rule $\F$ is {\em majority-preserving} if and only if for every agenda $\A$ and for every majority-consistent $\Dma$-profile $\Pf$ we have $\F(\Pf) =  ext(m(\Pf))$. 
A rule \gabri{$\F$ is} {\em weakly majority-preserving} if and only if  for every agenda $\A$ and for every majority-consistent $\Dma$-profile $\Pf$  we have $\F(\Pf) \supseteq ext(m(\Pf))$.
\end{definition}

Applied to the preference agenda with the transitivity constraint, majority-preserving coincides with the requirement that a social welfare function should return the pairwise majority ordering whenever it is transitive; applied to the \startj $W$ \endj constraint, it \jl{coincides with the requirement that a social welfare function should return the pairwise majority ordering whenever it has a dominating element, {\em i.e.}, whenever there is a Condorcet winner (which is slightly stronger than Condorcet-consistency).}

\begin{proposition}
$\RMSA$, $\RMCSA$, $\RMWA$, $\RRA$, \gabri{$\RLEXIMAX$}, $\RY$ and $\RMNAC$ are \linebreak majority-preserving. $\RDG$ and $\REVS$
are not even weakly majority-preserving.  
\end{proposition}

\begin{proof} Obvious cases are $\RMSA$, $\RMCSA$, $\RMWA$, $\RRA$, $\RLEXIMAX$, $\RY$ and $\RMNAC$. For $\REVS$,
which coincides with the Borda  rule for the preference agenda and the transitivity constraint, the result follows from the well-known fact that the Borda rule is not Condorcet-consistent.
%
%
%
For $\RDG$, consider the  profile $\Pf$ in Table~\ref{tab:pigdud}.
\begin{table}[h!]
\begin{center}
\begin{tabular}{r|cccccc}
Voters  & $\{p, $& $q,$ &$r,$ & $p\leftrightarrow q, $& $p \leftrightarrow r,$ & $q \leftrightarrow r\}$\\ \hline
$\Js_1, \Js_2 $ & + & + & + & + & + & + \\
$ \Js_3, \Js_4, \Js_5$ & - & + & + & - & - & + \\
$ \Js_6,\Js_7 $ & + & - & + & - & + & - \\
$ \Js_8, \Js_9$ & + & + & - & + & - & - \\
$\Js_{10}, \Js_{11}$ & - & - & - & + & + & + \\  \hline
$m(\Pf)$   & + & + & + & + & + & + \\  
\end{tabular}
\captionof{table}{A profile showing  $\RDG$ is not majority-preserving.}
\label{tab:pigdud}
\end{center}
\end{table}
\vspace{-5mm}

There are eight rational judgment sets over $\A$, \ie $|\Dma| = 8$. 
We check that for every $\Js, \Js' \in \Dma$, if $\Js \neq \Js'$ then $d_G(\Js,\Js')= 1$. Therefore,  $\ssum{\Js_i \in \Pf}d\gabri{_G}(\Js_3, \Js_i)=8$.  Now, $\leon{\ssum{\Js_i \in \Pf}}d_G(\Js,\Js_i)=9$  for \leon{every} 
$\leon{\Js} \in \{\Js_1, \Js_6, \Js_8, \Js_{10}\}$  and $\leon{\ssum{\Js_i \in \Pf}d_G(\Js,\Js_i)=}11$ for \leon{every} 
$\leon{J} \in (\Dma\setminus \{\Js_1, \Js_{3}, \Js_{6}, \Js_8,\Js_{10}\})$.
Therefore, $\RDG(\Pf) = \{\Js_3\}$ although $\Pf$ is majority-consistent and \srdjan{$m(P) = \Js_1$}.
\end{proof}

Let us call a pseudo-distance {\em non-degenerate} when there exist \ms{$\Js, \Js', \Js''$} such that $d(\Js, \Js'') > \max(d(\Js, \Js'), d(\Js', \Js''))$. Note that $d_H$ is non-degenerate.  


 

\begin{proposition} \label{prop:maxnomaj}For \leon{every} 
distance $d$, the 
rule $\F^{d,\textsc{max}}$ is not majority-preserving. If $d$ is non-degenerate then $\F^{d,\textsc{max}}$ is not weakly majority-preserving. 
\end{proposition}

\begin{proof} 
Let $\Js_1, \Js_2$  be two distinct judgment sets such that $d(\Js_1, \Js_2) \leq d(\Js,\Js')$ for all $\Js \neq \Js'$. 
Let $\Pf = \langle \Js_1, \Js_1, \Js_2\rangle$. $\Pf$ is majority-consistent, with \srdjan{$m(\Pf) = \Js_1$}, and yet $\F^{d,\textsc{max}}(\Pf)$ contains also $\Js_2$, therefore $\F^{d,\textsc{max}}$ is not majority-preserving. 
If moreover $d$ is non-degenerate, \srdjan{let $\Js_1, \Js_2, \Js_3$ be \jl{three} judgment sets such that $d(\Js_1, \Js_3) > \max(d(\Js_1, \Js_2), d(\Js_2, \Js_3))$.} Let  $\Pf = \langle \Js_1, \Js_2, \Js_3, \Js_3, \Js_3\rangle$. $\Pf$ is majority-consistent, with \srdjan{$m(\Pf) = \Js_3$}, and yet $\F^{d,\textsc{max}}(\Pf) = \{\Js_2\}$, therefore $\F^{d,\textsc{max}}$ is not weakly majority-preserving. 
\end{proof}

\begin{corollary}
$\F^{d_H,\textsc{max}}$ is not weakly majority-preserving.
\end{corollary}

\section{Inclusion relationships between the rules}\label{sec:relation}

We now establish the following (non)inclusion relationships between most of the rules introduced so far. 
As the case-by-case proof is long and not very interesting, we chose to have it in the Appendix. 

\begin{definition}
Given two judgment aggregation rules $\F_1$ and $\F_2$, we denote: 
\begin{itemize}
 \item $\F_1 \subseteq \F_2$ when $\F_1(\Pf) \subseteq \F_2(\Pf)$ holds  for every agenda $\A$ and every $\Dma$-profile $\Pf$. 
\item $\F_1 \subset \F_2$ when  $\F_1 \subseteq \F_2$ and $\F_1 \neq \F_2$.
\item $\F_1 \inc \F_2$ when neither $\F_1  \subseteq \F_2$ nor $\F_2 \subseteq \F_1$.
\end{itemize}
\end{definition}

\leon{Let $\F_1 \in \{\RDG, \REVS, \RMAX\}$ and $\F_2$ be majority-preserving. Note that $\F_1$ is not weakly majority-preserving, and that the counterexamples given in Section \ref{sec:mp} all have an odd $n$. If $n$ is odd \jl{(recall that $m(P)$ is then a complete judgment set)} then there is a majority-consistent profile $\Pf$ for which \jl{$m(\Pf) \notin \F_1(\Pf)$}, and
$\F_2(\Pf) = \jl{ \{m(P)\} }$.
This implies that $\F_1 \inc \F_2$.  Therefore, we have an incomparability relationship between \leon{a} 
rule in $\{\RDG, \REVS, \RMAX\}$ and \leon{a} 
rule in $\{\RMSA, \RMCSA, \RMWA, \RRA, \gabri{\RLEXIMAX},\RY, \RMNAC\}$. }

\begin{proposition} \label{prop:inc}
The inclusion and incomparability relations among the majority-preserving rules, and among the non majority-preserving rules, are represented on \marija{Tables}~\ref{tab:comp1} and ~\ref{tab:comp2}; a $\supset$ sign for row $\F_1$ and column $\F_2$ means that $\F_1  \supset \marija{\F_2}$, and an$\inc$sign, that $\F_1 \inc \F_2$.

\begin{table}[h!]
 \centering
\small
\begin{tabular}{|r||c|c|c|c|c|c| }\hline
                               &$\RMCSA$&$\RMWA$ & $\RRA$& $\RLEXIMAX$ & $\RY$  &    $\RMNAC$ \\ \hline
$\RMSA$         &$\supset $ &$\supset $ &$\supset$ & $\supset$ & $\inc$  &  $\inc$ \\ \hline
$\RMCSA$      &\multicolumn{1}{>{\columncolor[gray]{0.5}}c}{} 
&$\inc$ &$\inc$ & $\inc$ & $\inc$ &   $\inc$ \\ \hline
$\RMWA$          &\multicolumn{1}{>{\columncolor[gray]{0.5}}c}{}  
                                &\multicolumn{1}{>{\columncolor[gray]{0.5}}c}{} 
                                &\marija{$\inc$} &$\inc$ &$\inc$ & $\inc$ \\ \hline
$\RRA$              &\multicolumn{1}{>{\columncolor[gray]{0.5}}c}{} 
                                &\multicolumn{1}{>{\columncolor[gray]{0.5}}c}{} 
                                &\multicolumn{1}{>{\columncolor[gray]{0.5}}c}{} 
                                & $\supset$ & $\inc$ & $\inc$  \\ \hline
$\RLEXIMAX$            &\multicolumn{1}{>{\columncolor[gray]{0.5}}c}{} 
                               &\multicolumn{1}{>{\columncolor[gray]{0.5}}c}{} 
                               &\multicolumn{1}{>{\columncolor[gray]{0.5}}c}{}
                                &\multicolumn{1}{>{\columncolor[gray]{0.5}}c}{} &
                                $\inc$ & $\inc$  \\ \hline
$\RY$               &\multicolumn{1}{>{\columncolor[gray]{0.5}}c}{} 
                               &\multicolumn{1}{>{\columncolor[gray]{0.5}}c}{} 
                               &\multicolumn{1}{>{\columncolor[gray]{0.5}}c}{} 
                               &\multicolumn{1}{>{\columncolor[gray]{0.5}}c}{} 
                               &\multicolumn{1}{>{\columncolor[gray]{0.5}}c}{} 
                                &$\inc$ \\ \hline
\end{tabular}
 \caption{ (Non)inclusion relationships between the  majority-preserving rules.}\label{tab:comp1}
\end{table}

\begin{table}[h!]
 \centering

\begin{tabular}{|r||c|c|}\hline
                               &$\RDG$ & $\REVS$ \\ \hline
$\RMAX$            &$\inc $ &$\inc $ \\ \hline
$\RDG$      &\multicolumn{1}{>{\columncolor[gray]{0.5}}c}{} 
                     &   $\inc$ \\ \hline
\end{tabular}
 \caption{ (Non)inclusion relationships between the other rules.}\label{tab:comp2}
\end{table}
\end{proposition}


\section{Unanimity, monotonicity, homogeneity, reinforcement}\label{prop}

In preference aggregation, there are three classes of properties \cite{Zwicker13}: those that are satisfied by most common rules (such as neutrality or anonymity); those that are very hard to satisfy, and whose satisfaction, under mild additional condition, implies impossibility results; and finally, those that are satisfied by a significant number of rules and violated by another significant number of rules. Similarly, in judgment aggregation, weak properties such as anonymity are  satisfied by all our rules, while strong properties such as  
 independence are  violated by all our rules.
We have already studied an ``intermediate'' property: {\em \ms{majority-preservation}}. Here we consider four more:  {\em unanimity}, {\em monotonicity}, {\em homogeneity} and {\em reinforcement}.

\subsection{Unanimity}\label{dva}

Unanimity has been defined for {\em resolute rules} by Dietrich and List \cite{LiberalJA}: $\R$ is said to satisfy unanimity when for every  $\Dma$-profile $\Pf = \langle \Js_1, \ldots, \Js_n \rangle$ and \leon{every} 
$\ai \in \A$, if $\ai \in \Js_i$ for all $i  \leq n$, then $\ai \in \R(\Pf)$.\footnote{A weaker unanimity property has been defined by List and Puppe \cite{ListPuppe2009}, for resolute rules as well: $\R(\Pf) = \Js$ whenever all the voters  in $\Pf$ have the judgment set  $\Js$.}
We first generalise unanimity to irresolute rules, which gives us a weak and a strong version of unanimity. 

\begin{definition}[Weak and strong unanimity] 
Given $\ai \in \A$, the $\Dma$-profile $\Pf $ is said to be {\em $\ai$-unanimous} when $\ai \in \Js_i$ for every $\Js_i \in \Pf$. 
\begin{itemize}
\item $\F$ satisfies {\em weak unanimity} when for \leon{every} 
$\ai$-unanimous profile $\Pf$, there is a $\Js \in  \F(\Pf)$ such that  $\ai\in \Js$.
\item $\F$ satisfies {\em strong unanimity} when for \leon{every} 
$\ai$-unanimous profile $\Pf$, for all $\Js \in  \F(\Pf)$ we have $\ai\in \Js$.
\end{itemize}
\end{definition}

 
\begin{proposition}$\RMCSA$, $\RMWA$, $\RMAX$,  $\RMNAC$, $\RDG$ and $\REVS$ do not even satisfy weak unanimity.
\end{proposition}
\begin{proof}~
\begin{enumerate}
\item $\RMCSA$, $\RMWA$ and $\RMNAC$:
\marija{The proof that $\RMWA$  does not satisfy weak unanimity can be found in \cite{ADT09}. For $\RMCSA$ and $\RMNAC$}
consider the profile $\Pf$ from Table~\ref{tab:item5}.
\begin{table}[h!]
\centering
\begin{tabular}{r|cccccccccc}
{\small Voters}& $p$ & $p \rightarrow q \vee r$ & $q$ & $r$ & $p \rightarrow s \vee t$ & $s$ & $t$ & $p \rightarrow u \vee v$ & $u$ & $v$ \\\hline
$\Js_1$&+ & + & + & - & + & + & - & + & + & -\\
$\Js_2 $&+ & + & - & + & + & - & + & + & - & +\\
$\Js_3$&+ & -   & - & -  & -  & - & -   & -  & - & - \\ \hline
m(\Pf) & + & + & - & -  & + & - & - & +  & - & - \\ 
\end{tabular} \caption{A profile showing that  $\RMCSA(\Pf)$ and $\RMNAC(\Pf)$ do not satisfy weak unanimity.} \label{tab:item5}
\end{table}
$\RMCSA(\Pf)$, 
and $\RMNAC(\Pf)$ coincide: \srdjan{$\RMCSA(\Pf) = \RMNAC(\Pf)$ $=$ $ \{ \neg p, p \rightarrow(q \vee r), \neg q, \neg r, p \rightarrow (s \vee t), \neg s, \neg t, p \rightarrow (u \vee v) ,\neg u, \neg v \}$. $\RMNAC(\Pf) $ is obtained by reversing two $p$ judgments in either two of the three judgment sets of the profile.}

\item $\RMAX$:
Let $\PA=\{p,q,r,s, \alpha\}$, where  $\alpha = (p \wedge q \wedge r \wedge s) \vee (\neg p \wedge \neg q \wedge \neg r \wedge \neg s)$, and $\Pf=\langle \Js_1, \Js_2\rangle$ where
$\Js_1 = \{p,q,r,s, \alpha\}$ and 
$\Js_2 = \{\neg p, \neg q, \neg r, \neg s, \alpha\} $. 
$\RMAX(\Pf)$ selects all  $\Js\in \Dma$ for which $max(d_H(\Js,\Js_1), d_H(\Js,\Js_2)) =3$. For all such $\Js$ it holds that $\alpha \not\in \Js$, although there is unanimity on $\alpha$.


\item $\RDG$.
Let $\PA=\{ p_1, p_2, \ldots, p_{13}\}$ with $\Dma$ as given in Table~\ref{tab:rdgu}. 
\begin{table}[h!]
\begin{tabular}{r|ccccccccccccc}
Sets           &\{ $p_1 $,& $p_2$, &  $p_3$, &   $p_4$, &   $p_5,$ & $p_6$, &  $p_7$, &   $p_8$, &   $p_9$, & $p_{10}$, &  $p_{11}$, &   $p_{12}$ & $p_{13}$ \} \\\hline
   $ \Js^1$    &      +          &      -        &      -     &     +          &    -      &      -     &     +     &     -      &      -     &       +         &     -          &  -  &+   \\
   $ \Js^2$    &      +          &     +       &      -     &     +          &    +      &      -     &     +     &     +      &      -     &       +         &    +          &  -  &+   \\
   $ \Js^3$    &      -          &      +        &      -     &     -          &    +      &      -     &     -     &     +      &      -     &       -         &     +          &  -  &+   \\
   $ \Js^4$    &      -          &      +        &      +     &     -          &    +      &      +     &     -     &     +      &     +     &       -         &     +          & +  &+   \\
 $ \Js^5$    &     -              &      -        &      +    &     -          &    -      &     +    &    -     &     -      &      +     &       -         &     -          &  +  &+   \\
  $ \Js^6$    &    +              &      -        &      +    &    +          &    -      &     +    &    +    &     -      &      +     &      +         &     -          &  +  &+   \\
   $ \Js^7$    &      -          &      -        &      -     &     -          &    -      &      -     &     -     &     -      &      -     &       -         &     -          &  -  &-   \\
 \end{tabular}
 \caption{The   $\Dma$  for the example demonstrating that $\RDG$ does not satisfy weak unanimity.}\label{tab:rdgu}
 \end{table}

Consider the profile $\Pf = \langle\Js_1, \Js_2,\Js_3\rangle$, where $\Js_1 = \Js^1$, $\Js_2 = \Js^3$, and $\Js_3 = \Js^5$. We have that $\RDG(\Pf) = \{ \Js^7\}$, although there is unanimity on $p_{13}$. 
The $d_G$ distances between each set in $\Dmc$ are given in Table~\ref{tab:rdgscores}.
 
  \begin{table}[h!]
  \centering
 \begin{tabular}{r|ccccccc}
$d_G(.,.)$              &  $ \Js^1$  &  $ \Js^2$   &   $ \Js^3$   &     $ \Js^4$   &    $ \Js^5$   &   $ \Js^6$   &    $ \Js^7$  \\\hline
   $ \Js^1$    &      0          &      1             &      2             &    3               &    2                 &      1             &     1            \\
   $ \Js^2$    &     1          &     0              &      1            &    2               &    3                &      2             &     \srdjan{2}             \\
   $ \Js^3$    &      2          &     1             &     0             &     1              &   2                 &      3              &     1              \\
   $ \Js^4$    &     3          &      2              &      1            &     0               &    1                 &     2             &    \srdjan{2}                \\
 $ \Js^5$    &    2              &      3               &      2           &    1               &   0                  &     1               &   1                \\
  $ \Js^6$    &    1              &      2              &      3           &   2              &    1                   &    0               &    \srdjan{2}            \\
   $ \Js^7$    &      1        &     \srdjan{2}               &     1            &   \srdjan{2}              &   1                   &     \srdjan{2}               &    0        \\
 \end{tabular}\caption{The $d_G$ distances among the sets in  $\Dma$ from Table~\ref{tab:rdgu}.}\label{tab:rdgscores}.
 \end{table}

 \item $\REVS$. Consider a pre-agenda of $\PA=\{ p_1, p_2, \ldots, p_{13}\}$ with $\Dma$ as given in Table~\ref{tab:revs}.  Consider the profile $\Pf=\langle\Js_1, \Js_2, \Js_3\rangle$ where $\Js_1=\Js$, $\Js_2 = \Js'$,  and $\Js_3 = \Js''$. 
 \begin{table}[h!]
\begin{tabular}{r|ccccccccccccc}
Sets              &\{ $p_1 $,& $p_2$, &  $p_3$, &   $p_4$, &   $p_5,$ & $p_6$, &  $p_7$, &   $p_8$, &   $p_9$, & $p_{10}$, &  $p_{11}$, &   $p_{12}$ & $p_{13}$ \} \\\hline
   $ \Js^1$    &      +          &      -        &      -     &     +          &    -      &      -     &     +     &     -      &      -     &       +         &     -          &  -  &+   \\
   $ \Js^2$    &      -          &      +        &      -     &     -          &    +      &      -     &     -     &     +      &      -     &       -         &     +          &  -  &+   \\
 $ \Js^3$    &     -              &      -        &      +    &     -          &    -      &     +    &    -     &     -      &      +     &       -         &     -          &  +  &+   \\
   $ \Js^4$    &      -          &      -        &      -     &     -          &    -      &      -     &     -     &     -      &      -     &       -         &     -          &  -  &-   \\
 \end{tabular}
 \caption{The   $\Dma$  for the example demonstrating that $\REVS$ does not satisfy weak unanimity.}\label{tab:revs}
 \end{table}
 For      $ 1\leq i \leq 3$ and  $1\leq j \leq 13$,  we have that $\mbox{rev}(\Js_i, p_j) = 5$ and 
  $\mbox{rev}(\Js_i, \neg p_j) = 8$. We have that $\REVS(\Pf) = \{ \marija{\Js^4}\}$ since the score of $\marija{\Js^4}$ for $\Pf$ is 
	\srdjan{$192$}, 
	while the score of each of the profile judgment sets to $\Pf$ is 
	\srdjan{$163$}. 
\end{enumerate}
\end{proof}

This failure of $\REVS$ to satisfy unanimity is a surprising result, because the Borda social preference function (which ranks alternatives in a way consistent with their Borda scores)
satisfies Pareto-efficiency. 
 
\begin{proposition}\label{una:msa} $\RMSA$ satisfies weak unanimity but not strong unanimity.
\end{proposition}

\begin{proof} Let $\Pf$ be a $\ai$-unanimous $\Dma$-profile for some $\ai\in\A$. 
Note that for each $\psi \in m(\Pf)$,  there exists at least one $S \in \MC{m(\Pf)}$  such that $\psi \in S$. 
 Consequently there exists a judgment set in  $\RMSA(\Pf)$ that contains $\ai$.   

As a counter-example for $\RMSA$ satisfying strong unanimity,  consider   the profile $\Pf$ of Table~\ref{tab:item5} $\RMSA$ does not satisfy \ms{strong} unanimity since there exists $\Js \in \RMSA(P)$ such that $\neg p \in \Js$. Namely, 
$\{ \neg p, \neg (p \rightarrow(q \vee r)), \neg q, \neg r, \neg (p \rightarrow (s \vee t)), \neg s, \neg t , \marija{\neg (p \rightarrow (u \vee v)), \neg u, \neg v}\} \in \RMSA(\Pf)$\footnote{$\RMSA$ failing to satisfy strong unanimity is also a consequence of Theorem 2.2 in \cite{NehringPivatoPuppe2013}, which can be reformulated as: $\RMSA$ satisfies strong unanimity if and only if $\A$ does not contain a minimal inconsistent subset of size 3 or more.}.
\end{proof}

\begin{proposition} $\RRA$, $\RLEXIMAX$ and $\RY$ satisfy strong 
unanimity.
\end{proposition}
\begin{proof}
For $\RRA$ and $\RLEXIMAX$: Let $\Pf$ be a profile and $S \subseteq \A$ be the subset of the agenda consisting of all $\ai \in \A$ for which $\Pf$ is $\ai$-unanimous. Because individual judgment sets are consistent, the conjunction of all elements of $S$ is consistent. Now, when computing $\RRA(\Pf)$,  the elements of $S$ are considered first, and whatever the order in which they are considered, they are included in the resulting judgment set because no inconsistency arises. Therefore, for all $\ai \in S$ {and all $\Js \in \RRA(\Pf)$, we have $\ai \in \Js$.  Since $\RLEXIMAX \subset \RRA$, $\RLEXIMAX$ satisfies strong unanimity as well.}\\
For $\RY$: If $\ai$ is unanimously accepted in $\Pf$, it is also unanimously accepted in every majority-consistent subprofile of $\Pf$ 
and in its majoritarian judgment set.
\end{proof}

\subsection{Monotonicity}\label{sec:propmono}
In voting, monotonicity states that when the position of the winning alternative improves in some vote {\em ceteris paribus}, then this alternative remains the winner. 
We define below a generalisation of this property for (irresolute) judgment aggregation rules. It is a generalization of the monotonicity property defined by Dietrich and List \cite{DietrichList2005} for resolute rules.

\begin{definition}[Monotonicity]~\\
Let $\Pf,\Pf' \in \Dma^n$ be two profiles, and $\ai \in \A$.  $\Pf'$ is a $\ai$-improvement of  $\Pf$ when (a) $\Pf= (\Js_i, \Js_{-i})$, (b) $\Pf'= (\Js_i', \Js_{-i})$, (c) $\neg \ai \in \Js_i$, (d) $\ai\in \Js'_i$, and (e) for all $\psi \in \A$, $\psi \not\in\{ \varphi, \neg \varphi\}$, $\psi \in \Js_i$ if and only if  $\psi \in \Js'_i$. \jl{(Note that the definition implies that $\Js'_i$ is consistent, otherwise $\Pf'$ would not be a well-defined profile.)}
A judgment aggregation rule $\F$ is monotonic, when for every  $\Pf \in \Dma^n$ and its $\ai$-improvement $\Pf'\in\Dma^n$, for any $\ai \in \A$, it holds that: if $\ai \in \Js$ for every $\Js \in \F(\Pf)$, then $\ai \in \Js'$ for every $\Js' \in \F(\Pf')$.
\end{definition}

Note that not every profile has a $\ai$-improvement for a given $\ai \in \A$:
in Example \ref{ex:agex}, $\{p\wedge r, \neg q, \neg (p\wedge q)\}$ is a $p\wedge r$-improvement of $\{\neg (p\wedge r), \neg q, \neg (p\wedge q)\}$, but $\{p\wedge r,  q,  p\wedge q\}$ has no $\psi$-improvement\footnote{Recall that there is a constraint $\Ct=q\rightarrow r$ for the agenda in this example. } for \leon{every} 
$\psi \in \{\neg (p\wedge r), \neg q, \neg (p\wedge q)\}$.\medskip

%
%

In all the proofs of this section,  $\A$ is an agenda, $\ai$ an element of $\A$, $\Pf = (\Js_i, \Pf_{-i})$ a $\Dma$-profile, and $\Pf' = (\Js_i', \Pf_{-i})$ a  $\ai$-improvement of $\Pf$. \medskip

We start by proving the following lemmas, which will be useful for rules based on the majoritarian judgment set.

\begin{lemma}\label{lemma-mon1}
Let $P$ a profile and $P'$ a $\varphi$-improvement of $P$. Then \sj one of the following three statements is true: 
\begin{enumerate}
\item $m(P') = m(P)$;
\item $\neg \varphi \in m(P)$ and $m(P') = m(P) \setminus \{ \neg \varphi\}$;
\item $\varphi \notin m(P)$ and $m(P') = (m(P) \setminus \{ \neg \varphi\}) \cup \{ \varphi \}$. \ej
\end{enumerate}
\end{lemma}
\begin{proof} \sj
If $P'$ is a $\varphi$-improvement of $P$ then $N(P', \varphi) = N(P, \varphi) + 1$ and for all $\psi \in \A \setminus \{ \varphi, \neg \varphi\}$, $N(P', \psi) = N(P, \psi)$. Table~\ref{tab:presence} represents the different possible cases concerning the presence or not of $\varphi$ and $\neg \varphi$ in $m(P)$ and $m(P')$, and which of the three statements 1, 2 and 3 holds. Obviously, the columns $= \frac{n}{2} - 1$ and $=\frac{ n}{2}$ are relevant only when $n$ is even, and the column $= \frac{(n-1)}{2}$ is relevant only when $n$ is odd.
\begin{table}[h!]
\centering
 \begin{tabular}{c|c|c|c|c|c}
$N(P,\varphi)$ & $< \frac{n}{2}-1 $& $=  \frac{n}{2}-1$ &$ =  \frac{n-1}{2}$ & $=  \frac{n}{2} $& $ > \frac{n}{2}$ \\ \hline
$m(P) \cap \{\varphi,\neg \varphi\}$ & $\{ \neg \varphi \}$ & $\{\neg \varphi \}$ & $\{ \neg \varphi \} $& $\emptyset$ &$ \{ \varphi \}$ \\ \hline 
$N(P',\varphi)$ &$ < \frac{n}{2} $& $=  \frac{n}{2} $&$ =  \frac{n+1}{2}$ & $=  \frac{n}{2}+1$&$ > \frac{n}{2}$ \\ \hline
$m(P') \cap \{\varphi,\neg \varphi\}$ & $\{ \neg \varphi \} $& $\emptyset $&$ \{ \varphi \}$ & $\{ \varphi \} $&$ \{ \varphi \} $\\ \hline 
statement holding  & 1 & 2 & 3 & 3 & 1 \\ 
\end{tabular}
\caption{Different possible cases concerning the presence or not of $\varphi$ and $\neg \varphi$ in $m(P)$ and $m(P')$}\label{tab:presence}
\end{table}
In all cases, one of 1, 2 and 3 holds.
\ej
\end{proof}

\begin{lemma}\label{lemma-mon2}
Given a consistent
judgment set $\Js$, if every rational extension of $\Js$ contains $\varphi$, then (a) $\Js$ does not contain $\neg \varphi$ and (b)  $ext(\Js \cup \{ \varphi\}) = ext(\Js)$.
\end{lemma}

\begin{proof}
Assume that every rational extension of $\Js$ contains $\varphi$. A complete extension of $\Js$ containing $\neg \varphi$ would contain both $\varphi$ and $\neg \varphi$ and would not be consistent, hence (a) holds. For (b): because every rational extension  of $\Js$ contains $\varphi$, every rational extension of $\Js$ is also a rational extension of $\Js \cup \{\varphi\}$, and obviously a rational extension of $\Js \cup \{\varphi\}$ is also a rational extension of $\Js$%
\end{proof}

Lemma~\ref{lemma-mon-new} connects the monotonicity property with orders $\succ_\Pf$ that rank the judgment sets in $\Dma$ with respect to  given profile $\Pf \in \Dma^n$. We consider the rules $\F$ which select  as collective judgments for  $\Pf \in \Dma^n$ the undominated $\Js \in \Dma$ based on some order $\succ_\Pf$. This is  condition (c) in Lemma~\ref{lemma-mon-new}. Such rules  satisfy monotonicity when the order  $\succ_\Pf$ satisfies two properties. Intuitively, conditions (a) and (b) of Lemma~\ref{lemma-mon-new}, say that when going from $\succ_{\Pf}$ to $\succ_{\Pf'}$, judgment sets containing $\varphi$ (respectively $\neg \varphi$) can only move ``upward'' (respectively ``downward'') in the preference relation.


\startj
\begin{lemma}\label{lemma-mon-new}
Let $\F$ be a judgment aggregation rule such that there is a family of partial orders $(\succ_\Pf)_{\Pf \in \Dma^n}$ over $\Dma$ such that for every profile $\Pf$, every $\varphi$-improvement $\Pf'$ of $\Pf$, and all $\Js, \Js' \in \Dma$,
\begin{itemize}
\item[(a)] if [$\varphi \in \Js$ if and only if $\varphi \in \Js'$], then [$\Js \succ_\Pf \Js'$ implies  $\Js \succ_{\Pf'} \Js'$];
\item[(b)] if [$\varphi \in \Js$ and $\neg \varphi \in \Js'$], then [$\Js \succ_\Pf \Js'$ implies $\Js \succ_{\Pf'} \Js'$]; 
\end{itemize}
and such that 
\begin{itemize}
\item[(c)] $\F(\Pf)$ is the set of all $\Js \in \Dma$ such that there is no $\Js' \in \Dma$ with $\Js' \succ_\Pf \Js$. 
\end{itemize}
Then $\F$ satisfies monotonicity.
\end{lemma}

\begin{proof}
Assume $\F$ satisfies the conditions of the lemma.  Let $\Pf \in \Dma$, $\Pf'$ a $\varphi$ improvement of $\Pf$, and assume that (d) for all $\Js \in \F(\Pf)$ we have $\varphi \in \Js$.

Let $\Js' \notin \F(\Pf)$. From a repeated application of (c), we obtain that there is a $\Js \in \F(\Pf)$ such that  $\Js \succ_\Pf \Js'$. From (d), we have $\varphi \in \Js$. If $\varphi \in \Js'$ then from (a) and (c), we get $\Js \succ_{\Pf'} \Js'$; if $\varphi \notin \Js'$ then from (b) and (c), we get $\Js \succ_{\Pf'} \Js'$; therefore, in all cases, $\Js' \notin \F(\Pf')$.  We have shown that $\F(\Pf') \subseteq \F(\Pf)$, which together with (d) implies that for all $\Js \in \F(\Pf')$ we have $\varphi \in \Js$, from which we conclude that $\F$ satisfies monotonicity.
\end{proof}

 \begin{proposition} $\RMSA$, $\RMCSA$, $\RMWA$, $\RRA$ and $\RLEXIMAX$ satisfy monotonicity. 
 \end{proposition}
 \begin{proof}  
  In all cases, the proof comes from an application of Lemma \ref{lemma-mon-new}, with a suitable family of orders in each case.
 \begin{itemize}
\item for $\RMSA$, $\succ_\Pf$ is defined by $\Js \succ_\Pf \Js'$ if and only if $\Js \cap m(\Pf) \supset \Js' \cap m(\Pf)$.
 \item for $\RMCSA$, $\succ_\Pf$ is defined by $\Js \succ_\Pf \Js'$ if and only if $|\Js \cap m(\Pf)| > |\Js' \cap m(\Pf)|$. 
 \item for $\RMWA$, $\succ_\Pf$ is defined by $\Js \succ_\Pf \Js'$ if $\sum_{\psi \in \Js} N(\Pf,\psi) > \sum_{\psi \in \Js'} N(\Pf,\psi)$.
\item for $\RRA$, $\succ_\Pf$ is $>_\Pf^{\RRA}$ as defined in Section \ref{sec:rules}.
\item for $\RLEXIMAX$, $\succ_\Pf$ is $>^{leximax}_\Pf$ as defined in Section \ref{sec:rules}.
\end{itemize}
We give the proof that the conditions of the lemma are satisfied for $\RMSA$ (the other cases are similar). 

It comes directly from the definition of the rule that (c) holds. Let $P'$ be a $\varphi$-improvement of $P$. \sj Then one of the three conditions of Lemma \ref{lemma-mon1} holds. 
\begin{enumerate}
\item If $m(\Pf') = m(\Pf)$, then (a) and (b) obviously hold. 
\item Assume $\neg \varphi \in m(P)$ and $m(P') = m(P) \setminus \{ \neg \varphi\}$.  
Let $\Js, \Js'$ such that $\Js \succ_\Pf \Js'$, that is,
\begin{equation}\label{eq:A} 
 \Js \cap m(\Pf) \supset \Js' \cap m(\Pf).
 \end{equation}
\begin{itemize}
\item[(i)] if  $\varphi$ belongs to both $\Js$ and $\Js'$, then $\neg \varphi $ belongs to neither. Then $\Js \cap m(\Pf') = \Js \cap m(\Pf)$,  $\Js' \cap m(\Pf') = \Js' \cap m(\Pf)$, and  (\ref{eq:A}) implies $\Js \succ_{\Pf'} \Js'$.
\item[(ii)]  if $\neg \varphi$ belongs to both $\Js$ and $\Js'$, then $\Js \cap m(\Pf') = (\Js \cap m(\Pf)) \setminus \{\neg \varphi\} $ and  $\Js' \cap m(\Pf') = (\Js' \cap m(\Pf)) \setminus \{ \neg \varphi\}$, and then  (\ref{eq:A}) implies $\Js \cap m(\Pf') \supset \Js' \cap m(\Pf')$, that is, $\Js \succ_{\Pf'} \Js'$.
\item[(iii)]  assume $\varphi $ belongs to $\Js$ but not to $\Js'$; then  $\neg \varphi \notin \Js$, which together with $\neg \varphi \in m(P)$ and
(\ref{eq:A}) implies $\neg \varphi \notin \Js'$, therefore $\varphi \in \Js'$, contradiction.
\end{itemize}

\item Assume  $\varphi \notin m(P)$ and $m(P') = (m(P) \setminus \{ \neg \varphi\}) \cup \{ \varphi \}$.
\begin{itemize}
\item[(iv)]  if $\varphi $ belongs to both $\Js$ and $\Js'$,  then $\Js \cap m(\Pf') = (\Js \cap m(\Pf)) \cup \{ \varphi\} $ and  $\Js' \cap m(\Pf') = (\Js' \cap m(\Pf)) \cup \{ \varphi\}$. From  (\ref{eq:A}), $\varphi \notin  \Js \cap m(\Pf)$ and $\varphi \notin  \Js' \cap m(\Pf)$ we obtain $\Js \cap m(\Pf') \supset \Js' \cap m(\Pf')$, that is, $\Js \succ_{\Pf'} \Js'$.
\item[(v)]  if $\varphi $ belongs to neither $\Js$ and $\Js'$, then $\Js \cap m(\Pf') = \Js \cap m(\Pf)$ and  $\Js' \cap m(\Pf') \subseteq \Js' \cap m(\Pf)$, therefore $\Js \cap m(\Pf') = \Js \cap m(\Pf) \subset \Js' \cap m(\Pf) \subseteq \Js' \cap m(\Pf')$, therefore, $\Js \succ_{\Pf'} \Js'$.
\item[(vi)]  if  $\varphi $ belongs to $\Js$ but not to $\Js'$, 
then 
$\Js \cap m(\Pf') = (\Js \cap m(\Pf)) \cup \{ \varphi\} $ and  $\Js' \cap m(\Pf') \subseteq \Js' \cap m(\Pf)$, therefore, 
$\Js \cap m(\Pf') = (\Js \cap m(\Pf)) \cup \{ \varphi\} \supset (\Js' \cap m(\Pf)) \cup \{ \varphi\}  \supseteq \Js' \cap m(\Pf) = \Js' \cap m(\Pf')$, that is,  $\Js \succ_{\Pf'} \Js'$.
\end{itemize}
\end{enumerate}

(i), (ii), (v) and (vi) show that (a) holds in all cases, while (iii) and (vi) show that (b) holds in all cases.

\end{proof}
\endj

\begin{proposition}
$\RMAX$ and $\RDG$ satisfy monotonicity. 
\end{proposition}

\begin{proof}
We say that a distance $d$ satisfies {\em agreement monotonicity} \cite{JELIA2014} when for all judgment sets $\Js, \Js', \Js'' \in \Dma$, 
$(\Js'' \setminus \Js' )\subset (\Js'' \setminus \Js)$ implies $d(\Js',\Js'')\leq d(\Js,\Js'')$. Clearly, $d_H$ and $d_G$ are agreement monotonic \cite{JELIA2014}. Let \srdjan{$\star \in \{\sum, \max\}$}. For 
profile $\Pf = (\Js_i, i \in N)$ and judgment set $\Js$, define $d^\star(\Js,\Pf) = \star(d(\Js, \Js_i) \ | \ i \in N)$, and let $\F_d^\star$ the rule defined by $\F_d^\star(\Pf) = \mbox{argmin}_{\Js \in \Dma}d^\star(\Js,\Pf)$.

Let $d \in \{d_H, d_G\}$. Let $\Pf = (\Js_i, \Pf_{-i})$ be a profile such that $\varphi \in \Js$ for all $\Js \in \F_d^\star(\Pf)$, and $\Pf' = (\Js_i', \Pf_{-i})$ a $\varphi$-improvement of $\Pf$. 
Let $\Js \in \F_d^\star(\Pf)$ and $\Js'$ such that $\neg \varphi \in \Js'$. 
Because $\Js' \notin \F_d^\star(\Pf)$, we have $d^\star(\Js',\Pf) > d^\star(\Js,\Pf)$. 
Since $\neg \varphi \in \Js_i$, $\varphi \in \Js_i'$, and $\varphi \in \Js$, by agreement monotonicity of $d$ we have $d(\Js,\Js_i') \leq d(\Js,\Js_i)$ and $d(\Js',\Js_i') \geq d(\Js',\Js_i)$. 
Therefore, 
$$d^\star(\Js',\Pf') \geq d^\star(\Js',\Pf) > d^\star(\Js,\Pf) \geq d^\star(\Js,\Pf'),$$ 
which shows that $\Js' \notin \F_d^\star(\Pf')$. Therefore, $\F_d^\star(\Pf')$ satisfies monotonicity, and as particular cases, $\RMAX$, $\RMWA$ and $\RDG$ satisfy monotonicity. 
\end{proof}

\begin{proposition} $\RY$ satisfies
monotonicity.
 \end{proposition} 

\begin{proof}

\jl{Let $\Pf = (\Js_i, \Pf_{-i})$, and $\Pf' = (\Js_i', \Pf_{-i})$ a $\varphi$-improvement of $\Pf$, that is, 
(1) $\Pf' = (\Js_i', \Pf_{-i})$, with $\Js_i' = (\Js_i \setminus \{\neg \varphi\}) \cup \{\varphi\}$ (and $\Js'_i$ consistent).} Assume that (2) every judgment set in $\RY(\Pf)$ contains $\varphi$. Assume as well that (3) some judgment set in $\RY(\Pf')$ contains $\neg \varphi$, which means that (4) there is a maximum cardinality majority-consistent subprofile $\Qf'$ of $\Pf'$  such that $\neg \varphi \in \Js'$ for some $\Js' \in ext(m(\Qf'))$. 
We distinguish two cases: 
\medskip

\noindent {\em Case 1: $\Js'_i \notin \Qf'$}. Then $Q'$ is also a majority-consistent subprofile of $\Pf$. Because $Q'$ can be extended into a judgment set containing $\neg \varphi$, $Q'$ cannot be a maxcard majority-consistent subprofile of $P$, thus (5) there exists a maxcard majority-consistent subset $U$ of $P$ with $|U| > |Q'|$. 
\smallskip


\noindent {\em Case 1.1: $\Js_i \notin U$}. Then $U$ is also a majority-consistent subset of $P'$ with $|U| >  |Q'|$, which contradicts (4).\smallskip

\noindent {\em Case  1.2: $\Js_i \in U$}. Let $U' = (U\setminus \{\Js_i\}) \cup \{\Js_i'\}$. Note that $U'$ is a $\varphi$-improvement of $U$. Because of (2) and (5), we have (6) every rational extension of $m(U)$ contains $\varphi$. \sj By point (a) of Lemma \sj \ref{lemma-mon2} applied to $U$, $\neg \varphi \notin m(U)$. Now, we apply Lemma \sj \ref{lemma-mon1} to $U$. Condition (2) is impossible because $\neg \varphi \notin m(U)$, and condition (3) simplifies to $m(U') = m(U)\cup \{\varphi\}$; therefore, either \sj $m(U') = m(U)$ or $m(U') = m(U)\cup \{\varphi\}$. If $m(U') = m(U)$ then trivially, $U'$ is majority-consistent. If $m(U') = m(U)\cup \{\varphi\}$ then \ej applying \sj point (b) of \ej Lemma \ref{lemma-mon2} to $J = m(U)$, we obtain that 
$m(U)$ and $m(U')$ have the same rational extensions, which in turn implies that $U'$ is majority-consistent. Thus, $U'$ is a majority-consistent subset of $P'$ with $|U'| = |U| > |Q'|$, which contradicts (4). \medskip

\noindent {\em Case  2: $\Js'_i \in \Qf'$}. Let $Q = (\Qf' \setminus \{\Js_i'\}) \cup \{\Js_i\}$. 
Because of (4), $m(Q')$ does not contain $\varphi$, and because $Q'$ is a $\varphi$-improvement of $Q$, $m(Q)$ does not contain $\varphi$ either, and moreover $m(Q)$ and $m(Q')$ coincide on all issues other than $\varphi, \neg \varphi$. This implies that a rational extension $\Js'$ of $m(\Qf')$ containing $\varphi$ is also a rational extension of $m(Q)$, therefore, $m(Q)$ is consistent and has some rational extension containing $\neg \varphi$. This, together with (2), implies that $Q$ cannot be a maxcard majority-consistent subprofile of $P$, that is, there is a maxcard majority-consistent subprofile $T$ of $P$ such that $|T| > |Q|$. \smallskip

\noindent {\em Case 2.1: $\Js_i \notin T$}. Then $T$ is also a majority-consistent subprofile of $P'$, and $|T| > |Q| = |Q'|$, which contradicts (4).\smallskip

\noindent {\em Case  2.2: $\Js_i \in T$}. Let $T' = (T \setminus \{\Js_i\}) \cup \{\Js_i'\}$.
Similarly as in case 1.2, $m(T)$ and $m(T')$ have the same rational extensions, and $T'$ is a majority-consistent subprofile of $P'$ such that $|P'| > |Q'|$, which contradicts (4).\bigskip

\end{proof}

Here come now \sj two rather surprising results. \ej
\begin{table}[th!]
 \includegraphics[width=\textwidth]{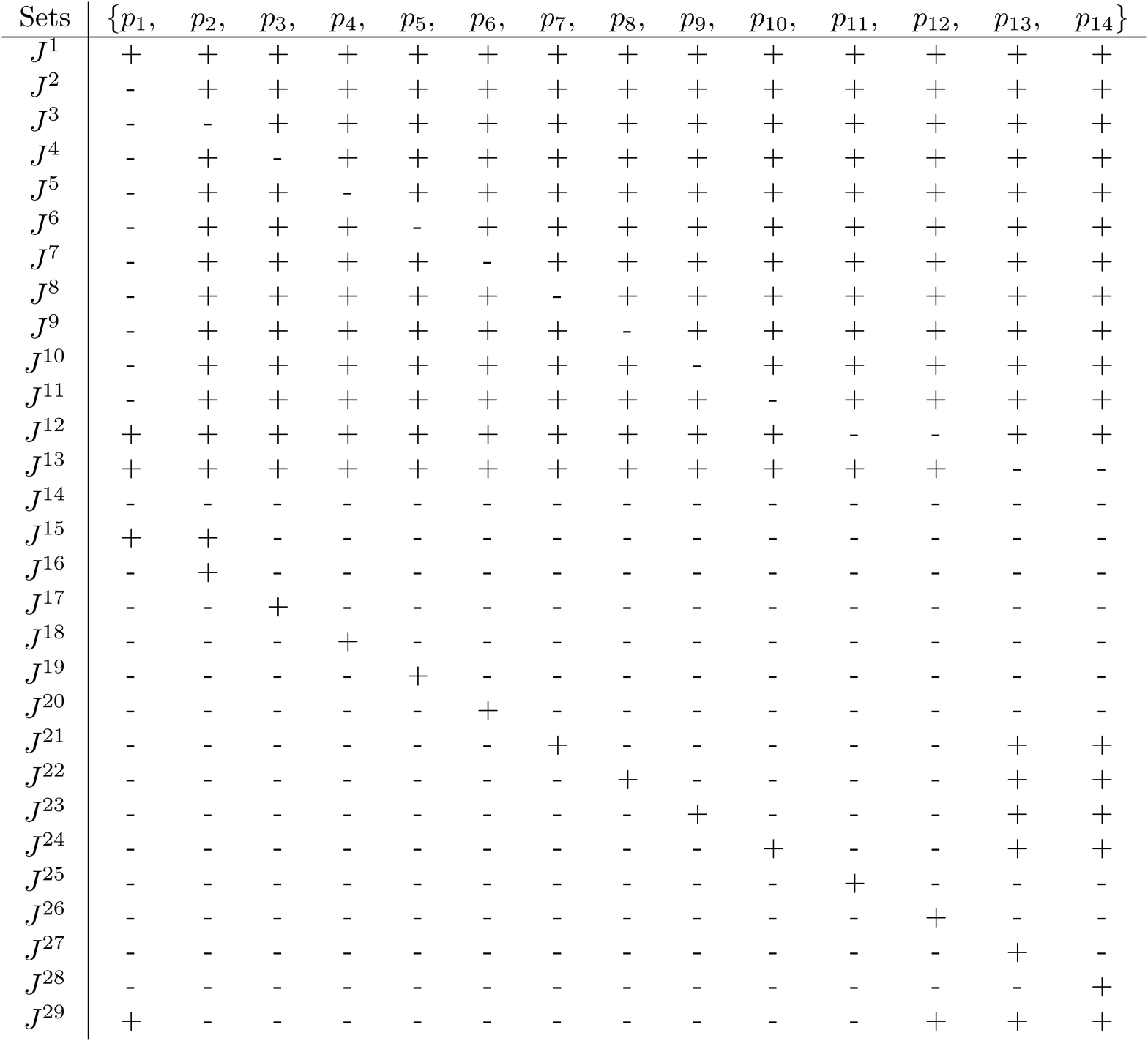}
\caption{Counterexample showing that $\REVS$ does not satisfying monotonicity. }\label{tab:borcount}
\end{table}

\begin{proposition} $\REVS$ does not satisfy monotonicity.
 \end{proposition}
\begin{proof} Consider an agenda $\A = \{p_1, \neg p_1, \ldots, p_{14}, \neg p_{14}\}$ and \srdjan{constraint} $\Ct$, such that $\Dmc_{\A}$ contains judgment sets $\Js^1$ to $\Js^{29}$ as given in Table~\ref{tab:borcount}.  

Consider the profile of two agents $\Pf=\langle \Js_1, \Js_2\rangle$ where $\Js_1=\Js^1$ and $\Js_2=\Js^{14}$. We obtain that $\REVS(\Pf)=\{ \Js^8,\Js^9, \Js^{10}, \Js^{11}\}$. Observe that for every $\Js \in \REVS(\Pf)$, $\neg p_1 \in \Js$. Consider the \jl{$\neg p_1$}-improvement of $\Pf$, the profile $\Pf' = \langle \Js'_1,\Js_2\rangle$, where $\Js'_1 = \Js^2$. We now have that $\REVS(\Pf')=\{ \Js^{29}\}$  and $\neg p_1 \not\in \Js^{29}$.
\end{proof}

Thus, $\REVS$ fails to satisfy monotonicity although, of course, the Borda rule does satisfy it.\footnote{This surprising result triggers further questions: are there interesting agendas, other than the preference agenda, for which $\REVS$ remains monotonic? Can we find a natural monotonic extension of the Borda rule? Such intriguing questions are left for further study.}

\begin{proposition} $\RMNAC$ does not satisfy monotonicity.
 \end{proposition}
\begin{proof} Consider an agenda $\A = \{p_1, \neg p_1, \ldots, p_{16}, \neg p_{16}\}$ and \srdjan{constraint} $\Ct$, such that \srdjan{$\Dmc_{\A}= \{ \Js^1, \ldots, \Js^9 \}$} as given \srdjan{in Table \ref{ce-mnac-1}}.
{\small
\begin{table}[htb!]
\begin{tabular}{c|c|ccc|ccc|ccc|ccc|ccc}
& $p_1$&$p_2$&$p_3$&$p_4$&$p_5$&$p_6$&$p_7$&$p_8$&$p_9$&$p_{10}$&$p_{11}$&$p_{12}$&$p_{13}$&$p_{14}$&$p_{15}$&$p_{16}$\\ \hline
$\Js^1$& +&+&+&+&+&+&+&+&+&+&+&+&+&+&+&+\\
$\Js^2$& +&+&-&-&+&-&-&+&-&-&+&-&-&+&-&-\\
$\Js^3$& -&-&+&-&-&+&-&-&+&-&-&+&-&-&+&-\\ \hline
$\Js^4$&+&+&+&-&+&-&+&+&-&-&+&+&-&+&+&-\\
$\Js^5$&-&+&+&-&+&+&-&+&+&-&+&+&-&+&+&-\\
$\Js^6$& -&+&-&-&+&+&+&+&+&+&+&+&+&+&+&+\\
$\Js^7$& -&+&+&+&+&-&-&+&-&-&+&-&-&+&-&-\\
$\Js^8$& -&-&+&-&+&-&+&+&-&-&-&+&-&-&+&-\\ \hline
$\Js^9$& +&-&+&-&-&+&-&-&+&-&-&+&-&-&+&-\\
\end{tabular}\medskip
\caption{Counter-example showing that $\RMNAC$ does not satisfy monotonicity. }
\label{ce-mnac-1}
\end{table}
}

Consider profile $\Pf=\langle \Js^1, \Js^2,\Js^3\rangle$. 
Let $Q = \langle \Js^1, \Js^2,\Js^8\rangle$.
We observe that $d(P,Q) = d(\Js^3,\Js^8) = 5$, and that $m(Q) = \{\Js_4\}$, therefore $Q$ is majority-consistent. We claim that there is no other majority-consistent profile $Q'$ with $d(P,Q') \leq 5$. To check this, we first give the distances between $\Js^1, \Js^2,\Js^3$ and other judgement sets; they are represented \srdjan{in Table \ref{ce-mnac-2}}.

\begin{table}[htb]%
$$\begin{array}{c|c|c|c|c|c|c|c|c|c}
& 				\Js^1	& \Js^2 &\Js^3 	&\Js^4 	&\Js^5 	&\Js^6 	&\Js^7 	&\Js^8 	&\Js^9 	\\ \hline
	\Js^1 & 0 		& 10 		& 11 		& 6 		& \srdjan{6} 		& 3 		& 9 		& \srdjan{10} 		& 10 		\\  
	\Js^2 & 10 		& 0 		& 11 		& 4 		& 6 		& 9 		& 3 		& 8 		& \srdjan{10}		\\  
	\Js^3 & 11 		& 11 		& 0 		& \srdjan{9} 		& 5 		& 10 		& 10 		& 5 		& 1
\end{array}$$
\caption{Distances between $\Js^1, \Js^2,\Js^3$ and other judgement sets.}
\label{ce-mnac-2}
\end{table}

Now, we list all profiles such that $d(P,Q') \leq 5$. There are 10; they are shown in \srdjan{Table \ref{ce-mnac-3}}, together with their distance to $P$ and their majoritarian aggregation.
\begin{table}%
$$\begin{array}{c|c|c}
Q' & d(P,Q') & maj(Q')\\ \hline
\Pf & 0 & + ~ + + - ~  +  + - ~ +  +  - ~+  +  - ~ + +  -\\
Q & 5 & J_4\\
\langle \Js^1, \Js^2,\Js^5\rangle & 5 & + ~ + + - ~  +  + - ~ +  +  - ~+  +  - ~ + +  - \\
\langle \Js^1, \Js^2,\Js^9\rangle & 1 &  + ~ + + - ~  +  + - ~ +  +  - ~+  +  - ~ + +  - \\ 
\langle \Js^1, \Js^7,\Js^3\rangle & 3 &  - ~ + + + ~  +  + - ~ +  +  - ~+  +  - ~ + +  - \\ 
\langle \Js^1, \Js^7,\Js^9\rangle & 4 &  + ~ + + + ~  +  + - ~ +  +  - ~+  +  - ~ + +  - \\
\langle \Js^1, \Js^4,\Js^3\rangle & 4 &  + ~ + + - ~  +  + + ~ +  +  - ~+  +  - ~ + +  - \\
\langle \Js^1, \Js^4,\Js^9\rangle & 5 &  + ~ + + - ~  +  + + ~ +  +  - ~+  +  - ~ + +  - \\
\langle \Js^6, \Js^2,\Js^3\rangle & 3 &  - ~ + - - ~  \srdjan{+  + -} ~ +  +  - ~+  +  - ~ + +  - \\
\langle \srdjan{\Js^6, \Js^2,\Js^9}\rangle & 4 &  + ~ + - - ~  \srdjan{+  + -} ~ +  +  - ~+  +  - ~ + +  - \\
\end{array}$$
\caption{All the profiles at distances $5$ or less from $P$.}
\label{ce-mnac-3}
\end{table}
All of them except  $Q$ are majority-inconsistent. Therefore, $\RMNAC(\Pf) = m(Q) = \{\Js^4\}$. Also, observe that $p_1 \in \Js^4$.\\

Consider now $P' = \langle \Js^1, \Js^2,\Js^9\rangle$. Since $\Js^3$ and $\Js^9$ differ only on $p_1$, $P'$ is a $p_1$-improvement of $P$.
Now, we claim that $\RMNAC(\Pf') = \{\Js^4, \Js^5\}$. First, we observe that $d(P',Q) = d(\Js^9,\Js^8) = 6$.  We show now that there are no majority-consistent profile $Q'$ with $d(P',Q') < 6$, and that there is another one with $d(P',Q') = 6$. To check this, we first give the distances between $\Js_9$ and each of the 9 consistent judgment sets, which are shown in \srdjan{Table \ref{ce-mnac-4}}.

\begin{table}%
$$\begin{array}{c|c|c|c|c|c|c|c|c|c}
& \Js^1& \Js^2 &\Js^3 &\Js^4 &\Js^5 &\Js^6 &\Js^7 &\Js^8 &\Js^9 \\ \hline
\Js^9 & \srdjan{10} & \srdjan{10} & 1 & \srdjan{8} & 6 & 11 & 11 & 6 & 0
\end{array}$$
\caption{Distances between $\Js_9$ and each of the 9 consistent judgment sets.}
\label{ce-mnac-4}
\end{table}

Let us now list all profiles such that $d(P',Q') \leq 6$. \srdjan{There are 14 of them}: all those that satisfied $d(P,Q') \srdjan{\leq 5}$, and four more that are shown \srdjan{in Table \ref{ce-mnac-5}}.

\begin{table}%
$$\begin{array}{c|c|c}
Q' & d(P',Q') & maj(Q')\\ \hline
\langle \Js^6, \Js^7,\Js^9\rangle & 6 & \Js^5 \\
\langle \Js^4, \Js^2,\Js^9\rangle & 6 &  + ~ + + - ~  +  - - ~ +  -  - ~ +  +  - ~ + +  - \\
\srdjan{\langle \Js^1, \Js^5,\Js^9 \rangle }& 6 &  + ~ + + - ~  + + - ~ + + - ~ + + - ~ + +  -  \\
\srdjan{\langle \Js^5, \Js^2,\Js^9 \rangle} & 6 &   + ~ + + - ~  + + - ~ + + - ~ + + - ~ + +  -  \\
\end{array}$$
\caption{There are 14 profiles at distance 6 or less from $P'$: four profiles from this table and all those from Table \ref{ce-mnac-3}.}
\label{ce-mnac-5}
\end{table}

We conclude that there are exactly two consistent profiles at distance 6 from $P'$: $Q$ and $\langle \Js^6, \Js^7,\Js^9 \rangle$.
Therefore, $\RMNAC(\Pf') = \{\Js^4, \Js^5\}$. Now, $\neg p_1 \in \Js^5$, which shows that $\RMNAC$ does not satisfy monotonicity.
\end{proof}



\subsection{Reinforcement}\label{sec:rein}

A social preference function $F$ satisfies {\em reinforcement} if whenever two profiles over disjoint electorates have some output rankings in common, then the profiles obtained by merging the two electorates leads to elect those rankings that are obtained for both profiles. This easily generalizes to judgment aggregation rules as follows.

\begin{definition}
For \leon{every} 
two profiles $\Pf = \langle \Js_1, \ldots, \Js_n \rangle$ and $\Qf = \langle \Js_{n+1}, \ldots, \Js_q \rangle$, we denote $\Pf + \Qf = \langle \Js_1, \ldots, \Js_q \rangle$. We say that a judgment aggregation rule $\F$ satisfies {\em reinforcement} when for \leon{every} 
agenda $\A$, and \leon{every} 
two profiles $\Pf$ and $\Qf$ over disjoint electorates, if $F(\Pf) \cap F(\Qf) \neq \emptyset$ then $F(\Pf+\Qf) = F(\Pf) \cap F(\Qf)$. 
\end{definition}

Young and Levenglick's theorem \cite{YoungLevenglick78} tells us that among social preference functions Kemeny's rule is the unique Condorcet extension satisfying neutrality and reinforcement. As a consequence, if a judgment aggregation rule is majority-preserving and if its application to a preference agenda defines a neutral social preference function, then this SPF has to be Kemeny's rule. 

\begin{corollary}
$\RMSA$, $\RMCSA$, $\RRA$, $\RLEXIMAX$, $\RMNAC$ and $\RY$ do not satisfy reinforcement. 
\end{corollary}

The following result does not come as a surprise, as reinforcement is the key property of scoring voting rules:

\begin{proposition}
All scoring rules satisfy reinforcement.
\end{proposition}

\begin{proof}
Let $\RS$ be the scoring rule based on some scoring function $s$. Let $P$ and $Q$ be \srdjan{two profiles} over disjoint electorates and assume that $\RS(\Pf) \cap \RS(\Qf) \neq \emptyset$. Let $\Js \in \RS(\Pf) \cap \RS(\Qf)$. Then $s(\Js,\Pf) \geq s(\Js', \Pf)$ and $s(\Js,\Qf) \geq s(\Js', \Qf)$ for all $\Js'$, therefore, $s(\Js,\Pf+\Qf) = s(\Js,\Pf) + s(\Js,\Qf) \geq s(\Js', \Pf) + s(\Js',\Qf) = s(\Js', \Pf+\Qf)$, which shows that $\Js \in \RS(\Pf+\Qf)$. Conversely, let $\Js' \notin \RS(\Pf) \cap \RS(Q)$. Without loss of generality, assume $\Js' \notin \RS(\Pf)$. Then, for \leon{every} 
$\Js \in \RS(\Pf) \cap \RS(\Qf)$ we have $s(\Js,\Pf) > s(\Js',\Pf)$, therefore $s(\Js',\Pf+\Qf) = s(\Js',\Pf) + s(\Js',\Qf) > s(\Js, \Pf) + s(\Js,\Qf) = s(\Js, \Pf+\Qf)$, which shows that $\Js' \notin \RS(\Pf+\Qf)$.
\end{proof}

\begin{corollary}
$\RMWA$ and $\REVS$ 
satisfy reinforcement. 
\end{corollary}

\begin{proposition}
$\RDG$ satisfies reinforcement.
\end{proposition}

The proof is similar to the proof above for scoring rules, replacing scores by distances and maximization by minimization. It would work more generally for \leon{every} 
rule \leon{minimising} 
the {\em sum} of distances to judgment sets. 

\begin{proposition}
$\RMAX$ does not satisfy reinforcement.
\end{proposition}

\begin{proof}
\srdjan{Let $[\A] = \{p, q, r\}$, $\Ct=\top$, $\Pf = \langle \{p,q,r\}, \{\neg p, \neg q, \neg r\} \rangle$ and $\Qf = \langle \{\neg p, q, r \} \rangle$. $\RMAX(\Pf) = \{ \{\neg p, q, r\}, \{p, \neg q, r\}, \{p, q, \neg r\}, \{\neg p, \neg q, r\}, $\linebreak $ \{\neg p, q, \neg r\}, \{\neg p, \neg q, r\} \}$ and $\RMAX(\Qf) = \{ \{\neg p, q, r \} \}$, therefore \linebreak $\RMAX(\Pf) \cap \RMAX(\Qf) = \RMAX(\Qf)  \neq \emptyset$. However, $\RMAX(\Pf + \Qf) = \RMAX(\Pf)$.  }
\end{proof}

A similar negative result holds more generally for \leon{every} 
rule \leon{minimising} 
the {\em maximum} of distances to judgment sets.  However, such rules, including $\RMAX$, satisfy this weak version of reinforcement: if $\RMAX(\Pf) \cap \RMAX(\Qf) \neq \emptyset$,   then  
 $\RMAX(\Pf + \Qf) \cap (\RMAX(\Pf) \cap \RMAX(\Qf)) \neq \emptyset.$

\subsection{Homogeneity}\label{sec:homo}

Let us write $k\Pf$ for $\underbrace{\Pf+\cdots +\Pf}_{k~ times}$, where $+$ has been defined in Subsection \ref{sec:rein}.

\begin{definition} A judgment aggregation rule $\F$ satisfies homogeneity when for every $\Dma$-profile $\Pf$ and positive integer $k$, it holds that $\F(k\Pf) = \F(\Pf)$.
\end{definition}

Homogeneity being weaker than reinforcement, we already know that it is satisfied by all scoring functions (including $\RMWA$ and $\REVS$)
and by $\RDG$. 

\begin{proposition}
Every judgment aggregation rule based on the majority set satisfies homogeneity. 
\end{proposition}
\begin{proof}
Let $\F$ be a judgment aggregation rule based on majority set. Since for every profile $\Pf$ and every $k \in \{1, 2, 3, \ldots\}$ we have that $m(\Pf) = m (k \Pf)$ then $\F(\Pf) = \F(k \Pf)$. 
\end{proof}

\begin{corollary} $\RMSA$ and $\RMCSA$ satisfy homogeneity.
\end{corollary}

\begin{proposition} $\RRA$ and $\RLEXIMAX$ satisfy homogeneity.
\end{proposition}
\begin{proof}
$\RRA(\Pf)$ is fully determined by the weak order on $\A$ induced by the values of $N(P,.)$. Since for every $k \in \{1, 2, \ldots\}$ and \leon{every} 
$\varphi, \psi \in \A$, $N(kP,\varphi) \geq N(kP,\psi)$ if and only if $N(P,\varphi) \geq N(P,\psi)$, 
we have $\RRA(\Pf) = \RRA(k \Pf)$. The same proof works also for $\RLEXIMAX$.
\end{proof}

\begin{proposition} 
$\RY$ does not satisfy homogeneity.
\end{proposition}

\startj
\begin{proof}
This is a consequence of the fact that the Young voting rule does not satisfy homogeneity (see Example 2 in \cite{Young77}).
\end{proof}
\endj

\begin{proposition} $\RMNAC$ does not satisfy homogeneity.
\end{proposition}
\begin{proof}
Consider the profile $\Pf$ from \marija{Table} \ref{tab:mnac-homo}.
\begin{table}[h!]
\begin{center}\begin{tabular}{r|cccc}
Voters & $p \wedge r$ & $p \wedge s$ & $q$ & $p \wedge q$ \\ \hline
$\Js_1,\Js_2,\Js_3 $& + & + & +  & +\\
$\Js_4, \Js_5,\Js_6 $& + & + & -   & - \\
$\Js_7-\Js_{10}$ & - & - & + & -\\
$\Js_{11}$ & - & - & - & -\\ \hline
$m(\Pf)$              & +  & + & + & -    
\end{tabular}
\captionof{table}{Profile showing that $\RMNAC$ does not satisfy homogeneity.}
 \label{tab:mnac-homo}
\end{center}
\end{table}
We have $\RMNAC(\Pf) = \{ \{ \neg (p \wedge r), \neg (p \wedge s), q, \neg (p \wedge q) \},  
\{ (p \wedge r), (p \wedge s), \neg q, \neg (p \wedge q) \} \}$.
Consider now the profile $\Pf' = 2 P = P + P$. We have $\RMNAC(\Pf') = \{ \{ \neg (p \wedge r), \neg (p \wedge s), q, \neg (p \wedge q) \}  \}$.
\end{proof}
 
\begin{proposition} For \leon{every} 
distance $d$, $\F^{d,max}$ satisfies homogeneity. 
\end{proposition}

\begin{proof}
For \leon{every} 
profile $\Pf$, \startj judgment set $\Js$, and positive integer $k$, we have $\mmax{\Js_i \in k\Pf}~d(\Js,\Js_i) = \mmax{\Js_i \in \Pf}~d(\Js,\Js_i)$.
The result follows.
\end{proof}

As a consequence, $\F^{d_H,max}$ satisfies homogeneity. 

\section{Summary}\label{sec:conclu}

We have listed a number of existing judgment aggregation rules, and for a number of important properties we have identified those rules that satisfy it. 
These properties come in four groups: (1) \ms{majority-preservation}; (2) weak and strong unanimity; (3) monotonicity, and (4) reinforcement and homogeneity (weaker than reinforcement). 
Our results are summarized in Table~\ref{tab:summr}. 

 \definecolor{LightCyan}{rgb}{0.88,1,1}
\begin{table}[h!]
\centering
 \includegraphics[width=\textwidth]{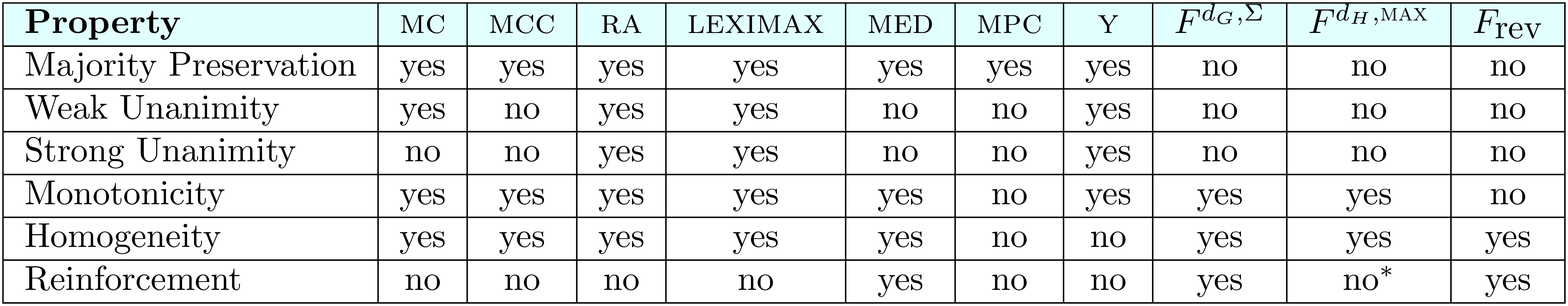}
\caption{Summary of rules and properties they (do not) satisfy.}
\label{tab:summr}
\end{table}

We may use this table to derive a tentative dominance relation between rules. \jl{Let ${\cal P}$ be the set of properties considered here, that is, ${\cal P}$ = \{majority preservation, weak unanimity, strong unanimity,  monotonicity, reinforcement, homogeneity \}.} Say that a rule $\F$ \jl{${\cal P}$-dominates} a rule $\F'$ \ms{when} the set of properties \jl{in ${\cal P}$} satisfied by $\F$ strictly contains the set of properties \jl{in ${\cal P}$} satisfied by $\F'$. (Of course, this is somewhat arbitrary because many other properties could have been considered; but still, these are among the most important properties.) Then:

\begin{itemize}
\item $\RRA$ and $\RLEXIMAX$  \jl{${\cal P}$-dominate} $\RMSA$, $\RMCSA$, $\RMNAC$ and $\RY$; 
\item $\RMWA$  \jl{${\cal P}$-dominates} $\RMCSA$, $\REVS$, $\RMNAC$, $\RDG$ and $\RMAX$.
\end{itemize} 

This leaves us with three  \jl{${\cal P}$-undominated} rules, coming in two groups: $\RMWA$, and the very closely related rules $\RRA$ and $\RLEXIMAX$. Given the importance of the median rule, which generalizes the Kemeny rule, it should not come as a surprise that $\RMWA$ performs well. The presence of $\RRA$ and $\RLEXIMAX$ on this podium is somewhat more surprising.

Some rules have been left out of this study. Importantly, we did not consider quota-based rules, because the quota has to be chosen {\em depending on the agenda} to ensure that the judgment sets are consistent, which prevents the use of a quota-based rule in an agenda-independent way. This is even more patent for premise-based and conclusion-based rules.  

Finally, a challenging question is the axiomatization of some judgment aggregation rules, or some families of rules, for which our work can be seen as a very first (and very incomplete) step.\footnote{\jl{There exists a recent axiomatization of the median rule (in the general judgement aggregation framework) \cite{NehringPivato16}; we are not aware of axiomatic characterizations of other rules.}}

\section*{Acknowledgments}
Gabriella Pigozzi and Srdjan Vesic benefited from the support of the project AMANDE ANR-13-BS02-0004 of the French National Research Agency (ANR). J\'{e}r\^{o}me Lang  benefited from the support of the ANR project  14-CE24-0007-01 CoCoRICo-CoDec. The authors would like to thank   Denis Bouyssou, as well as  anonymous reviewers.

\bibliographystyle{spmpsci}      
\bibliography{SCWbiblio}   
 
\section*{Appendix: Proof of Proposition \ref{prop:inc}.}

\startj Many of the non-inclusion relationships can be derived from the profile of our running example, introduced in 
Example \ref{ex:maj}, and used again in Example \ref{ex:msa} for $\RMCSA$ and $\RMSA$, Example \ref{ex:mwa} for $\RMWA$, Example \ref{ex:ra} for $\RRA$ and $\RLEXIMAX$, Example \ref{ex:y} for $\RY$, and Example \ref{ex:mnac} for $\RMNAC$. This profile already shows that $\RMSA \not \subseteq \RMCSA$, $\RMSA \not \subseteq \RMWA$, $\RMSA \not \subseteq \RRA$, $\RMSA \not \subseteq \RLEXIMAX$, $\RY \not \subseteq \RRA$,   \ms{$\RY \not \subseteq \RLEXIMAX$},  that $\RMWA$ and $\RRA$ are incomparable, as well as \ms{$\RMWA$ and $\RLEXIMAX$}, that $\RY$ and each of $\RMSA$, $\RMCSA$, $\RMWA$ are incomparable, \ms{and that $\RMNAC$ is incompatible with each of $\RY$, $\RRA$, and $ \RLEXIMAX$.}

The inclusion relationships $\RMCSA \subseteq  \RMSA$, $\RLEXIMAX \subseteq \RRA$ are clear from their definitions, and a proof that $\RMWA \subseteq \RMSA$ can be found in \cite{PuppeNehring2011}. 

We now prove what remains to be proven:\endj

\begin{enumerate}
\item $\RRA \subseteq \RMSA$:
If $\Js \in \RRA(\Pf)$ then, by definition of $\RRA$, $\Js \cap m(\Pf)$ is a maximal consistent subset of $m(\Pf)$, thus $\Js \in \RMSA(\Pf)$. 
\item $\RRA \not\subseteq \RLEXIMAX$: Consider the profile $\Pf$ in Table~\ref{tab:ralex}.
\begin{center}\begin{tabular}{r|cccccc}
Voters & $p \wedge q$ & $p$ & $q$ & $p \wedge r$ & $q \wedge r $& $s$\\ \hline
$\Js_1 -\Js_5 $& - & + & -  & +& - & +\\
$\Js_6 - \Js_{10}$& - & - & +   & - & + & -\\
$\Js_{\marija{11}} - \Js_{14}$ & + & + & + & +& + & +\\
$\Js_{15}$ & + & + & + & - & - & -\\ \hline
$m(\Pf)$              & -  & + &+ & + & + & +
\end{tabular}
\captionof{table}{A profile showing that $\RRA \not\subseteq \RLEXIMAX$.}
 \label{tab:ralex}
\end{center}
$\RRA(\Pf) = \{ \{ p \wedge q, p, q, p \wedge r, q \wedge r, s\}, 
                            \{ \neg(p \wedge q), p, \neg q, p \wedge r, \neg (q \wedge r), s\}
                             \{ \neg (p \wedge q), \neg p, q, \neg (p \wedge r), q \wedge r, s\}\}$ 
and
{\sc leximax}$(\Pf) = \{ \{ p \wedge q, p, q, p \wedge r, q \wedge r, s\} \}$. 
 

\item $\RMWA \not \subseteq \RMCSA$: Consider the example from 
Table~\ref{tab:item5}. 
We have \srdjan{$\RMCSA(\Pf) = \{  \neg p, p \rightarrow(q \vee r), \neg p, \neg r, p \rightarrow (s \vee t), \neg s, \neg t ,p \rightarrow (u \vee v) ,\neg u, \neg v\}$} and
$\{ p, p \rightarrow(q \vee r), \neg q, r, p\rightarrow (s \vee t), \neg s, t, \marija{p \rightarrow (u \vee v) ,\neg u, v}  \} \in \RMWA(\Pf)$. 
\item $\RRA$ and $\RLEXIMAX$ are incomparable with  $\RMCSA$.\\
Consider again the example from Table~\ref{tab:item5}. $\RMCSA(\Pf) = \{ \{ \neg p, p \rightarrow(q \vee r), \neg p, \neg r, p \rightarrow (s \vee t), \neg s, \neg t \}, \marija{p \rightarrow (u \vee v) ,\neg u, \neg v}\}$ and for every $\Js \in \RRA(\Pf)$, and {\em a fortiori} for every $\Js \in \RRA(\Pf)$,  $p \in \Js$. Thus $\RMCSA \not \subseteq \RRA$ and $\RLEXIMAX \not \subseteq \RMCSA$.

\item $\RRA \not \subseteq \RY$: Consider the example from Table \ref{tab:RAvY}. 
\begin{table}[h!]
\centering
 \begin{tabular}{r|cccccccc}
Voters &$ \{p,$ & $q,$ & $p ~\wedge~q,$ & $r,$ & $s,$ & $r ~\wedge~s,$ &$t\}$\\ \hline
$\Js_1$ & + & + & + & - & + & - & +\\
$\Js_2,\Js_3,\Js_4$ & + & + & + & - & + & - & -\\
$\Js_5-\Js_8$& + & + & + & + & - & - & -\\
\rowcolor{light-gray}$\Js_9,\Js_{10}$ & + & - & - & + & - & - & -\\
$\Js_{11}-\Js_{14}$ & + & - & - & + & + & + & +\\
$\Js_{15} -\Js_{18}$ & - & + & - & + & + & + & +\\\hline \hline
\marija{Rule} &$ \{p,$ & $q,$ & $p ~\wedge~q,$ & $r,$ & $s,$ & $r ~\wedge~s,$ &$t\}$\\\hline
\marija{$\RRA(\Pf)$} & + & + & + & + & + & + & + \\
                        &+ & + &+ & + & + & + &  -  \\\hline
\marija{ $\RY(\Pf)$}  & + & + & + & + & + & + & +                      
 \end{tabular}\caption{A profile showing that  $\RRA \not \subseteq \RY$.} \label{tab:RAvY}
\end{table}
The minimal number of voters to remove  to make the profile majority-consistent is two. These two voters are the two voters of the fourth row (light gray shaded).  
We have $\RY(\Pf) = \{ \{ p, q, p \wedge q, r, s, r \wedge s, t\} \}$ and
$\RRA(\Pf) = \{ \{ p, q, p \wedge q, r, s, r \wedge s, t\}, \{ p, q, p \wedge q, r, s, r \wedge s, \neg t\} \}$. 
Thus, $\RRA \not \subseteq \RY$.


\item  $\RMNAC$ is incomparable with $\RMSA$:
Consider the pre-agenda $\PA=\{ p,q,p\wedge q,p\wedge\neg q\wedge r,\alpha_1,\alpha_2, \neg p\wedge q \wedge s,\alpha_3,\;\alpha_4, \alpha_5,\alpha_6, \alpha_7 \}$, where\\
\noindent$\;\alpha_1 = p\wedge\neg q  \wedge r \wedge\neg q$, 
$\alpha_2~=~p\wedge\neg q  \wedge r \wedge\neg q\wedge\neg q$,
 $\alpha_3 = q\wedge\neg p\wedge\neg p \wedge s$,
 $\alpha_4 = q\wedge\neg p\wedge\neg p\wedge\neg p \wedge s$,
 $ \alpha_5 = (p \leftrightarrow q) \wedge t$,
  $ \alpha_6 = (p \leftrightarrow q) \wedge t \wedge t$ and
    $ \alpha_7 = (p \leftrightarrow q) \wedge t \wedge t \wedge t$.
 Let $\Pf$ be the profile from Table~\ref{tab:msamwa}. 
\begin{table}[h!]
\centering
 \includegraphics[width=\textwidth]{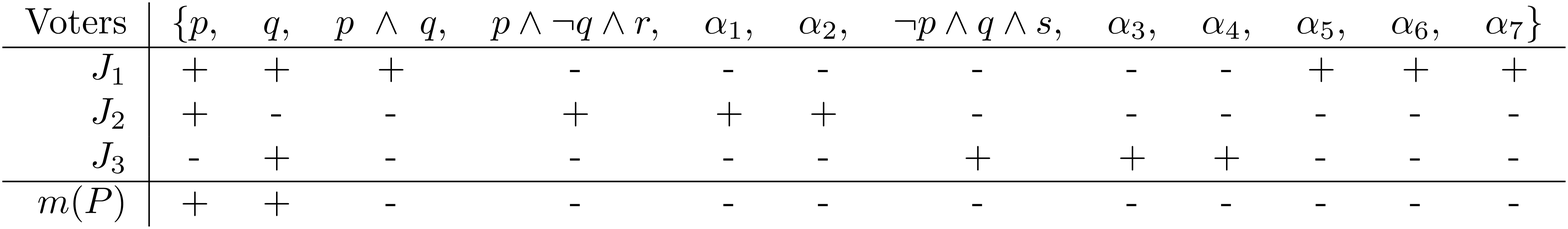}
\caption{ A profile $\Pf$ showing that $\RMNAC \inc \RMSA$.} \label{tab:msamwa}
\end{table}

We obtain \\ 

\hspace{-0.8cm}
{\centering
\includegraphics[width=\textwidth]{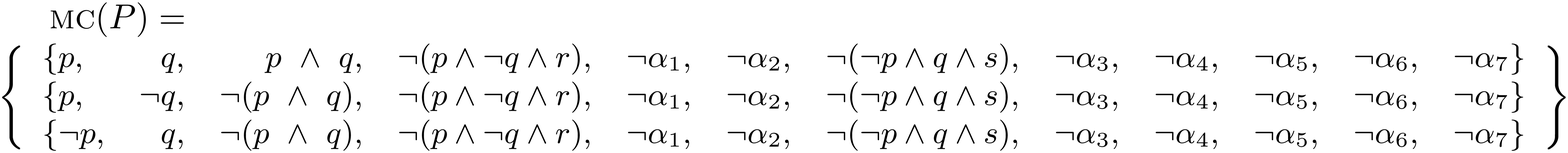}}

\begin{table}[h!]
\begin{center}
 \includegraphics[width=\textwidth]{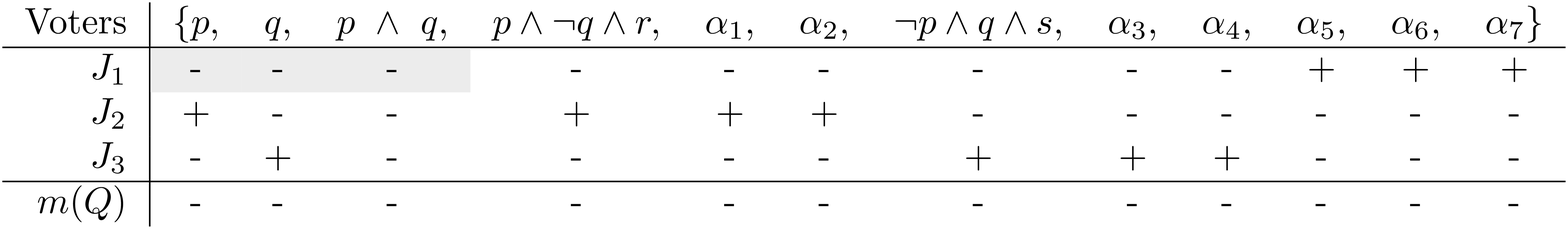}
\caption{\marija{A profile $\Qf$ at  a minimal $D_H$ sitance from $\Pf$ in Table~\ref{tab:msamwa}. Colouered are the cells with the judgments reversed from $\Pf$}} \label{tab:msamwa1}
\end{center}
\end{table}
 
To obtain $\RMNAC(\Pf)$, we need to change the first three judgments of the first voter, obtaining the profile \marija{$\Qf$} given in Table~\ref{tab:msamwa1}. This is the minimal change, since if either the second or the third agent change either their judgment on $p$ or their judgment on $q$, they have to change additional other three judgments. We obtain  $\RMNAC(\Pf)=\{  \{ \neg p,  \neg q, \neg(p~\wedge~q),\neg(p\wedge\neg q \wedge r),\neg \alpha_1, \neg \alpha_2, \neg (p\wedge  q \wedge s), \neg \alpha_3, \neg \alpha_4,\neg \alpha_5, \neg \alpha_6, \neg \alpha_7\} \}$.
Thus, $\RMSA \inc \RMNAC$.

\item $\RMNAC$ is incomparable with $\RMCSA$:
\marija{Consider the profile $\Pf$ from Table \ref{tab:mnac-homo-double}. We have $\RMNAC(\Pf) = \{ \{ \neg (p \wedge r), \neg (p \wedge s), q, \neg (p \wedge q) \}  \}$ and $\RMCSA(\Pf) = \{ \{ p \wedge r, p \wedge s, \neg q, \neg (p \wedge q) \}, \{ p \wedge r, p \wedge s, q, p \wedge q \} \}$; thus $\RMNAC(\Pf) \cap \RMCSA(\Pf) = \emptyset$.}
\marija{
\begin{table}[h!]
\begin{center}\begin{tabular}{r|cccc}
Voters & $p \wedge r$ & $p \wedge s$ & $q$ & $p \wedge q$ \\ \hline
$\Js_1-\Js_6 $& + & + & +  & +\\
$\Js_7-\Js_{12} $& + & + & -   & - \\
$\Js_{13}-\Js_{20}$ & - & - & + & -\\
$\Js_{21}, \Js_{22}$ & - & - & - & -\\ \hline
$m(\Pf)$              & +  & + & + & -    
\end{tabular}
\captionof{table}{A profile showing that $\RMNAC$ does not satisfy homogeneity.}
 \label{tab:mnac-homo-double}
\end{center}
\end{table}}

\item $\RMNAC \inc \RMWA$: Consider the pre-agenda $\PA=\{ p,q,p~\wedge~q,p~\wedge~\neg q,\alpha_1,\alpha_2,q\wedge\neg p,\alpha_3,\;\alpha_4 \}$, where $\;\alpha_1 = p\wedge\neg q\wedge\neg q$, 
$\alpha_2~=~p\wedge\neg q\wedge\neg q\wedge\neg q$,
 $\alpha_3 = q\wedge\neg p\wedge\neg p$ and
 $\alpha_4 = q\wedge\neg p\wedge\neg p\wedge\neg p$.

\begin{table}[h!]
\centering
\begin{tabular}{c|ccccccccc}
Voters &$\{ p,$ & $q,$ & $p~\wedge~q,$ & $p~\wedge~\neg q,$ & $\alpha _1,$ & $\alpha_2,$ & $q\wedge\neg p,$ & $\alpha_3,$ & $\alpha_4\}$  \\ \hline
  $ \Js_1$ &+ &+ &+ &- &- &- &- &- &-  \\
  $\Js_2$&+ &-  &-  &+ &+ &+ &- &- &-  \\
  $\Js_3$&- &+  &-  &- &- &- &+ &+ &+   \\ \hline \hline
Rules &$\{ p,$ & $q,$ & $p~\wedge~q,$ & $p~\wedge~\neg q,$ & $\alpha _1,$ & $\alpha_2,$ & $q\wedge\neg p,$ & $\alpha_3,$ & $\alpha_4\}$  \\ \hline
  $ \RMNAC$ &- &- &- &- &- &- &- &- &-  \\\hline
  $\RMWA$ &+&+ &+ &- &- &- &- &- &-  
\end{tabular}\caption{A profile  showing that  $\RMNAC \inc \RMWA$.} \label{tab:mnacinc}
\end{table}

\marija{
We obtain $\RMNAC(\Pf)$
by changing the first three judgments of the first voter. This is the minimal change, since if either the second or the third agent change either their judgment on $p$ or their judgment on $q$, they have to change additional other three judgments.  }
Observe that $\RMWA(\Pf) = \{\{ p, q, p \wedge q, \neg (p \wedge \neg q), \neg \alpha_1, \neg \alpha_2, \neg (q \wedge \neg p), \neg \alpha_3, \neg \alpha_4 \}\}$ since for this judgment set the weight is  17, and for the remaining three other possible judgment sets the weights are: 14  for the set of the judgment sets of the second, and third voter and 16 for the judgment set \linebreak $\{\neg p, \neg q, \neg(p~\wedge~q),\neg(p~\wedge~\neg q),\neg \alpha_1, \neg \alpha_2, \neg (q\wedge\neg p),\neg \alpha_3, \neg \alpha_4\}$. \\
Thus $\RMNAC \inc \RMWA$.

\item $\RMAX$, $\REVS$ and $\RDG$ are pairwise incomparable: 
We give one counterexample for all three pairs. Let $\A_{\Alt}$ be the preference agenda for the set of alternatives $\Alt = \{c_1,c_2,c_3,c_4\}$, together with the transitivity constraint.  
Consider the profile given in Table~\ref{fig:incnmp}. The collective judgments obtained by $\REVS$,  $\RDG$, and $\RMAX$ are represented in the last five rows of this table. 
\begin{table}[h!]
\centering
\begin{tabular}{r|cccccc}
Voters& $c_1Pc_2$ & $c_1Pc_3$ & $c_1Pc_4$ & $c_2Pc_3$ & $c_2Pc_4$ & $c_3Pc_4$\\\hline
$\Js_1$ &+& + & +& - & - & -  \\
$\Js_2, \Js_3$ & - & + & +& + & + & + \\\hline\hline
Rule & $c_1Pc_2$ & $c_1Pc_3$ & $c_1Pc_4$ & $c_2Pc_3$ & $c_2Pc_4$ & $c_3Pc_4$\\\hline
$\REVS(\langle \Js_1, \Js_2, \Js_3\rangle)$              &+ &+ & +& + & + & + \\
$\RDG(\langle \Js_1, \Js_2, \Js_3\rangle)$               &- &+ & +& + & + & + \\ 
$\RMAX(\langle \Js_1, \Js_2, \Js_3\rangle)$            &+ &+ & +& - & + & + \\
                               &+ &+ & +& + & - & + \\
                               &+ &+ & +& + & + & - \\
\end{tabular}
\caption{A profile showing that $\RMAX$, $\REVS$ and $\RDG$ are mutually incomparable.}
\label{fig:incnmp}
\end{table}
 
\end{enumerate}

\end{document}